\pdfoutput=1
\documentclass{article}

\usepackage{microtype}
\usepackage{graphicx}
\usepackage{subfigure}
\usepackage{booktabs} %

\usepackage[accepted]{icml2021}

\usepackage{amsmath,amsfonts,bm}

\def\1{\bm{1}}

\DeclareMathAlphabet{\mathsfit}{\encodingdefault}{\sfdefault}{m}{sl}
\SetMathAlphabet{\mathsfit}{bold}{\encodingdefault}{\sfdefault}{bx}{n}

\newcommand{\E}{\mathbb{E}}

\newcommand{\R}{\mathbb{R}}

\newcommand{\sigmoid}{\sigma}

\DeclareMathOperator*{\argmax}{arg\,max}

\usepackage{dsfont}
\usepackage{adjustbox}
\usepackage{thm-restate}
\usepackage{url}
\numberwithin{equation}{section}
\usepackage{wrapfig}
\usepackage{multirow,makecell}
\usepackage{placeins}

\def \argmax {\textrm{argmax}}

\def \be {\begin{eqnarray}}
\def \en {\end{eqnarray}}

\def \xv {\mathbf{x}}
\def \yv {\mathbf{y}}

\def \E {\mathds{E}}

\newcommand{\dom}{\mathop{\mathrm{dom}}}

\newcommand{\supp}{\mathop{\mathrm{supp}}}

\def \g {\gamma}

\def \hh {\hat{h}}

\newcommand{\Hc}{\mathcal{H}}

\newcommand{\Xc}{\mathcal{X}}

\def \labelingf{f}
\def \D {\Delta}

\def \one {\mathbf{1}}
\def \Ps{P_{\textrm{s}}} %
\def \Pt{P_{\textrm{t}}} %

\def \psdensity{p_{\textrm{s}}} %
\def \ptdensity{p_{\textrm{t}}} %

\def \pzsdensity{p_{\textrm{s}}^{\textrm{z}}} %
\def \pztdensity{p_{\textrm{t}}^{\textrm{z}}} %

\def \pzsdistribution{P_{\textrm{s}}^{\textrm{z}}} %
\def \pztdistribution{P_{\textrm{t}}^{\textrm{z}}} %

\def \hypot{\mathcal{H}}

\def \riskTlh{R^{\ell}_T(h)}
\def \riskSlh{R^{\ell}_S(h)}
\def \domainin{\mathcal{X}}
\def \domainout{\mathcal{Y}}
\def \sourcedataset{\textrm{S}}
\def \targetdataset{\textrm{T}}
\def \latentspace{\mathcal{Z}}

\def \hypotf{h}
\def \fdiv{\textrm{D}_{\phi}}
\def \fHdiscrepancy{\textrm{D}^{\phi}_{\hypot}}
\def \fhHdiscrepancy{\textrm{D}^{\phi}_{h,\hypot}}

\def \JShHdiscrepancy{\textrm{D}^{\textrm{JS}}_{h, \hypot}}

\def \TVHdiscrepancy{\textrm{D}^{\textrm{TV}}_{\hypot}}
\def \TVhHdiscrepancy{\textrm{D}^{\textrm{TV}}_{h, \hypot}}

\usepackage{booktabs}

\usepackage{todonotes}
\usepackage{natbib}
\usepackage{amsthm}
\usepackage{amssymb, amsmath}
\usepackage{xcolor}
\usepackage{graphicx}
\usepackage{hyperref}

\hypersetup{colorlinks=true, allcolors=blue, pdftitle={f-Domain-Adversarial Learning: Theory and Algorithms},
pdfsubject={f-Domain-Adversarial Learning: Theory and Algorithms}
}
\usepackage{cleveref}
\usepackage{thmtools}
\usepackage{thm-restate}
\usepackage{xpatch}

\newcommand{\updated}[1]{ #1 }

\newtheorem{definition}{Definition}

\newtheorem{prop}{Proposition}

\newtheorem{lem}{Lemma}

\usepackage{framed}
\colorlet{shadecolor}{gray!30}

\icmltitlerunning{$f$-Domain-Adversarial Learning: Theory and Algorithms }

\begin{document}

\twocolumn[
\icmltitle{$f$-Domain-Adversarial Learning: Theory and Algorithms}

\begin{icmlauthorlist}
   \icmlauthor{David Acuna}{nv,to,vect}
   \icmlauthor{Guojun Zhang}{wo,vect}
   \icmlauthor{Marc T. Law}{nv}
   \icmlauthor{Sanja Fidler}{nv,to,vect}
\end{icmlauthorlist}

   \icmlaffiliation{nv}{NVIDIA}
   \icmlaffiliation{to}{University of Toronto}
   \icmlaffiliation{wo}{University of Waterloo}
   \icmlaffiliation{vect}{Vector Institute}

   \icmlcorrespondingauthor{David Acuna}{davidj@cs.toronto.edu,dacunamarrer@nvidia.com}

   \icmlkeywords{Machine Learning,Domain Adaptation,Domain Adversarial Learning, ICML}

\vskip 0.3in
]

\printAffiliationsAndNotice{}  %

\begin{abstract}
Unsupervised domain adaptation is used in many machine learning applications %
where, during training, a model
has access  
to unlabeled data in the target domain, and a related labeled dataset. 
In this paper, we introduce a novel and general domain-adversarial framework.
Specifically, we derive a novel generalization bound for domain adaptation that exploits a new measure of discrepancy between distributions based on a variational characterization of $f$-divergences. 
It recovers the theoretical results from \citet{ben2010theory} as a special case,
and supports divergences 
 used in practice. 
Based on this bound, we derive a new algorithmic framework that 
 introduces  a key correction in the original adversarial training method of \citet{ganin2016domain}. 
We show that many regularizers and ad-hoc objectives introduced over the last years in this framework are then not required to achieve performance comparable to (if not better than) state-of-the-art domain-adversarial  methods.
Experimental analysis conducted on real world natural language and computer vision datasets show that our framework outperforms existing baselines,  and obtains the best results for $f$-divergences that were not considered previously in domain-adversarial learning.
\end{abstract}

\section{Introduction}

The ability to learn new concepts from general-purpose data and transfer them to related but different contexts
is critical in  many modern applications. 
One such prominent scenario is called \textit{unsupervised domain adaptation}.
In domain adaptation, the learner has access to both a small (unlabeled) dataset on its domain of interest, and to a larger labeled dataset on a domain related to the target domain but with different distribution. %
The model is  trained with
both the labeled and unlabeled datasets, and  
it is
expected to generalize well to the target dataset if the gap between both domains is not very significant.

The paramount importance of domain adaptation (DA) has led to remarkable advances in the field.
From a theoretical point of view, 
~\cite{ben2007analysis,ben2010theory,david2010impossibility,mansour2009domain} 
provided generalization bounds for unsupervised DA based on discrepancy measures that are a reduction of the Total Variation (TV).
\citet{zhang19bridging} recently 
proposed the  Margin Disparity Discrepancy (MDD) with the aim of closing the gap between theory and algorithms. 
Their notion of discrepancy is tailored to  margin losses and builds on the observation 
of only taking a single supremum over the class set to make optimization easier.
Theories based on weighted
combination of hypotheses for multiple source DA have also been developed~\citep{hoffman2018algorithms}.

From an algorithmic perspective in the context of neural networks,~\citet{ganin2015unsupervised,ganin2016domain} proposed the idea of learning domain-invariant representations 
as an adversarial game.
This approach led to a plethora of methods 
 including state-of-the-art approaches 
 such as ~\citet{shu2018dirt, long2018conditional,hoffman2018cycada,zhang19bridging}. 
Although these methods were   
explained with insights from the theory of~\citet{ben2010theory}, and more recently through MDD~\citep{zhang19bridging}, 
both the $\hypot \Delta \hypot$ divergence 
 \citep{ben2010theory} 
and MDD are hard to optimize with deep neural networks.  \textit{Ad-hoc} objectives have thus been introduced to minimize the divergence between the source and target distributions in a common representation space. 
This has led to a disconnect between theory and the current SoTA practical methods. 
Specifically, the  domain-classifier from~\citet{ganin2016domain} that gives rise to domain-adversarial training methods is inspired by the proxy $\mathcal{A}$-distance from \citet{ben2007analysis} which itself is an approximation of the empirical estimation of the $\hypot \Delta \hypot$-divergence. 
It has been shown however that the discrepancy being minimized in practice in this framework corresponds to the JS-divergence \cite{ganin2015unsupervised}. 
Nonetheless, to the best of our knowledge, no clear connection between the DA theory and the algorithms that are typically employed has been made, i.e.~generalization bounds for DA with $f$-divergences have not been derived.

\textbf{Contributions.} 
In this paper, we derive a more general domain adaptation generalization bound  based on a variational characterization of $f$-divergences. 
These allow us to clearly  connect domain-adversarial training methods with the domain adaptation theory from an $f$-divergence minimization perspective. 
The theoretical results from \citet{ben2010theory} can be seen as a special case of our work for a specific choice of divergence. %
For the Jensen-Shannon (JS) divergence, we show how to rectify the original domain-adversarial training method from \citet{ganin2016domain}. 
Our analysis shows that after a key correction, many regularizers and ad-hoc objectives introduced in the DANN framework are not required to achieve performance comparable to (if not better than) state-of-the-art unsupervised domain adaptation methods that rely on adversarial learning. 
We also study how learning invariant representations for different choices of divergence affects the transfer performance on real-world datasets. 
In particular, the choice of the Pearson $\chi^2$ divergence is sufficient to outperform previous methods without additional techniques and/or additional hyperparameters.

\vspace{-3mm}
\section{Preliminaries}\label{sec:prem}
In this paper, we focus on the unsupervised domain adaptation task. During training, we assume that the learner has access to a source dataset %
of $n_s$ \emph{labeled} examples $\sourcedataset=\{(x^s_i, y^s_i)\}^{n_s}_{i=1}$, and a target dataset of 
$n_t$ \emph{unlabeled} examples $\targetdataset=\{(x^t_i)\}^{n_t}_{i=1}$, where the source datapoints $x^s_i$ are sampled i.i.d.~from a distribution $\Ps$ (source distribution) over the input space $\domainin$ and the target inputs $x^t_i$ are sampled i.i.d.~from a distribution $\Pt$ (target distribution) over $\domainin$. 
Usually, 
in the case of binary classification, we have $\domainout = \{0, 1\}$ and
in the multiclass classification  scenario,  $\domainout = \{1, . . . , k\}$. 
When the definition of $\domainin$ or $\domainout$ cannot be inferred from the context, we will mention it explicitly.

We denote a labeling function as $\labelingf : \domainin \to \domainout $, and use indices $\labelingf_s$ and $\labelingf_t$ to refer to the source and target labeling functions, respectively.
The task of unsupervised domain adaptation is to find a hypothesis function $\hypotf  : \domainin \to \domainout$ that generalizes to the target dataset $\targetdataset$ (i.e., to make as few errors as possible by comparing with the ground truth label $\labelingf_t (x_i^t)$).
 The risk of a hypothesis $\hypotf$ w.r.t. the labeling function $\labelingf$, using a loss function $\ell:\domainout \times \domainout \to \R_+$ under distribution $\mathcal{D}$ is defined as: $R^{\ell}_\mathcal{D}(h,\labelingf):=\E_{x \sim \mathcal{D}} [\ell(h(x),\labelingf(x))]$. We also assume that $\ell$ satisfies the triangle inequality.
For simplicity of notation, we define $R^{\ell}_S(h):=R^{\ell}_{P_s}(h,\labelingf_s)$ and $R^{\ell}_T(h):=R^{\ell}_{P_t}(h,\labelingf_t)$ where the indices $S$ and $T$ refer to the source and target domains, respectively.  %
In the stochastic scenario, we let the labeling function be the optimal Bayes classifier i.e $\labelingf (x)=\argmax_{\hat y \in \domainout  } P(y=\hat y |x)$ \cite{mohri2018foundations}. $P(y|x)$ denotes the class conditional distribution for either the source ($P_s(y|x)$) or the target domain ($P_t(y|x)$), respectively.
The empirical risks over the source dataset $\sourcedataset$ and the target dataset $\targetdataset$ are denoted by $\hat R_S$ and $\hat R_T$.

\textbf{Comparing domains with $f$-divergences.} 
A key component of domain adaptation is to study the discrepancy between the source and target distributions. 
In our work, 
we define new discrepancies between source and target distributions 
based on the variational characterization of popular choices of $f$-divergences.
Thus, we start by providing the definition of $f$-divergences.
\begin{definition} [\textbf{$f$-divergence, \citet{csiszar1967information, ali1966general}}] \label{def:fdivergence} 
Let $\Ps$ and $\Pt$ be two distribution functions with densities $\psdensity$ and $\ptdensity$, respectively. Let  $\psdensity$ be absolutely continuous w.r.t $\ptdensity$ and both be absolutely continuous with respect to a base measure $dx$. Let $\phi:\R_{+} \to \R$  be a convex, lower semi-continuous function %
that satisfies $\phi(1)=0$. The $f$-divergence $D_\phi$ is defined as:~
    \begin{equation} \label{eq:fdivergence}
  D_\phi(\Ps || \Pt)=\int \ptdensity(x) \  \phi\left( \frac{\psdensity(x)}{\ptdensity(x)} \right) dx.
  \end{equation}
\end{definition}
\begin{table*}[tb] 
\caption{  Popular $f$-divergences, their conjugate functions and choices of $a$. }
\begin{center}
\footnotesize
	\begin{tabular}{lllllll}
	\toprule
	\centering
	Divergence & $\phi(x)$  & Conjugate  $\phi^*(t)$ & $ \phi'(1) $& Activation func. $a(x)$ %
	\\ \midrule
		Kullback-Leibler (KL)
	& $x \log x$
	& $\exp(t-1)$
	& $1$
	& $x$
	\\
	Reverse KL (KL-rev)
	& $- \log x$
	& $-1-\log (-t)$
	& $-1$
	&  $-\exp x$
	\\
	Jensen-Shannon (JS)
	& $- (x+1) \log \frac{1+x}{2}+x \log x$
	& $- \log(2-e^t)$
	& $0$
	& $\log \frac{2}{1+\exp(-x)}$
	\\
	 Pearson $\chi^2$ 
	& $(x-1)^2$ 
	& $t^2/4 + t$ 
	& $0$ 
	& $x$
	\\
		Total Variation (TV) 
	& $\frac{1}{2}|x-1|$
	& $ \one_{-1/2\leq t \leq 1/2}$
	& $[-1/2, 1/2]$
	& $\frac{1}{2}\tanh x$\\
	\bottomrule

	\end{tabular}
\end{center}

\label{tbl:choices_f_diver}
\vspace{-3mm}
\end{table*}

\vspace{-2mm}
\textbf{Variational characterization of $f$-divergences.}  \citet{nguyen2010estimating} derive a general variational method that estimates $f$-divergences from  samples by turning the estimation problem into  variational optimization. They show that any  $f$-divergence can be written as (see details in \Cref{app:divergences}):
\begin{equation}
D_\phi(\Ps || \Pt) \geq \sup_{T \in \mathcal{T} } \E_{x\sim \Ps} [T(x)] - \E_{x\sim \Pt}[\phi^{*}(T(x))] \label{eq:sup_measurable_function}
\end{equation}

where $\phi^{*}$ is the (Fenchel) conjugate function of $\phi:\R_{+} \to \R$  defined as $\phi^*(y):=\sup_{x \in \R_{+}} \{ xy- \phi(x) \}$, and  $T: \domainin \to \dom \phi^*$. 
The equality holds if $\mathcal{T}$ is the set of all measurable functions. 
Many popular divergences that are heavily used in machine learning and information theory are special cases of $f$-divergences.  We summarize them and their conjugate function in Table~\ref{tbl:choices_f_diver}. 
For simplicity, we assume in the following that $\domainin \subseteq \R^n$ and each density (i.e $\psdensity$ and $\ptdensity$)
is absolutely continuous.

\vspace{-3mm}
\section{Discrepancies and Generalization Bounds}

Domain adaptation bounds generally build upon the idea of bounding the gap between the source and target domains' error functions in terms of the discrepancy between their probability distributions. 
 We first remind the reader of the seminal work of \citet{ben2010theory} that bounds the risk of any binary classifier in the hypothesis class $\hypot$ with the following theorem:
\begin{restatable}{thm}{TotalVar}  %
\label{theo:riskf_uninformative}

 If $\ell(x,y)=|h(x)-y|$ and $\hypot$ is a class of functions, then for any $h \in \hypot$ we have:
\begin{equation}
\begin{aligned}
&\riskTlh \leq  {\riskSlh} + D_{\rm TV}(\Ps\|\Pt) \\& \ + { \min \{ \E_{x \sim \Ps}[|\labelingf_t(x)-\labelingf_s(x)|],  \E_{x \sim \Pt}[|\labelingf_t(x)-\labelingf_s(x)|] \} }.
\end{aligned}
\label{eqn:bound_riskf_uninformative}
\end{equation}
\end{restatable}
Here, $$D_{\rm TV}(\Ps\|\Pt) := \sup_{T \in \mathcal{T} } | \E_{x\sim \Ps} [T(x)] - \E_{x\sim \Pt}[T(x)] |$$ is the TV and $\mathcal{T}$ is the set of measurable functions. TV is an $f$-divergence such that $\phi(x) = |x - 1|$ in Definition~\ref{def:fdivergence}. For any function $\phi(x) \geq |x - 1|$, one can replace $D_{\rm TV}(\Ps\|\Pt)$ in 
Eq.~\eqref{eqn:bound_riskf_uninformative} 
with $D_{\phi}(\Ps\|\Pt)$. \Cref{theo:riskf_uninformative} thus bounds a classifier’s target error in terms of the source error, the divergence between the two domains, and the dissimilarity of the labeling functions. 
Unfortunately, $D_{\rm TV}(\Ps\|\Pt)$ cannot be estimated from finite samples of arbitrary distributions \citep{kifer2004detecting}. 
It is also a very loose upper bound
as it involves the supremum over all measurable functions and does not account for the hypothesis class.

\subsection{Measuring discrepancy with $f$-divergences}

In the previous section, we have shown that measuring the similarity between  $\Ps$ and $\Pt$ is critical in the derivation of generalization bounds and/or the design of algorithms. 
We now introduce a new discrepancy called $\fHdiscrepancy$ 
that aims to generalize previous results to the family of $f$-divergences while solving the two
aforementioned problems, namely \textbf{(1)} estimation of the divergence from finite samples of arbitrary distributions (\Cref{lemma_fhh_from_finite_samples})  and \textbf{(2)} restriction  of the discrepancy  to the set including the hypothesis class $\hypot$.
 (Defs.~\ref{def:disc1} and \ref{def:disc2}). 
In Section~\ref{sec:f-divergence_bound} we show 
 how this allows us to extend the bounds studied in \citet{ben2010theory}.

\begin{restatable}[\textbf{$\fHdiscrepancy$ discrepancy}]{definition}{}
	\label{def:disc1}
 Let $\phi^*$ be the Fenchel conjugate of  a convex, lower semi-continuous function $\phi$ that satisfies $\phi(1)=0$, and
let $\mathcal{\hat T}$ be a set of measurable functions 
such that $\mathcal{ \hat  T} = \{ \ell(h(x),h'(x)) : h,h' \in \hypot \}$. 
We define the discrepancy 
between $\Ps$ and $\Pt$ as:~
\begin{equation}\label{eq:def_fH}
\begin{aligned}
\fHdiscrepancy(\Ps || \Pt):= &\sup_{ h,h' \in \hypot  }  | \E_{x\sim \Ps} [\ell(h(x),h'(x))]- \\& \E_{x\sim \Pt}[\phi^{*}(\ell(h(x),h'(x)))  |.
\end{aligned}
\end{equation}
\end{restatable}

The $\fHdiscrepancy$  discrepancy can be interpreted as a lower bound estimator of a general class of $f$-divergences (\Cref{lem:upper_lower_bound_fdiv}). Therefore, for any hypothesis class $\hypot$ and choice of $\phi$, $\fHdiscrepancy$ is never larger than its corresponding $f$-divergence. 
In Lemma~\ref{lemma_fhh_from_finite_samples} we show that its computation can be bounded in terms of finite examples. %
Finally, we recover the $\hypot \Delta \hypot$-divergence \citep{ben2010theory} if we consider $\phi^*(t)=t$ and $\ell(h(x), h'(x)) = \one[h(x)\neq h'(x)]$, which is  {the TV}.
\begin{restatable}[\textbf{$\fhHdiscrepancy$ discrepancy}]{definition}{}
	\label{def:disc2}
 Under the same conditions as above, the discrepancy between two distributions $\Ps$ and $\Pt$ is defined by:~  
\begin{equation} \label{eq:proposed_discrepancy}
\begin{aligned}
\fhHdiscrepancy(\Ps || \Pt):= &\sup_{ h' \in \hypot  } | \E_{x\sim \Ps} [\ell(h(x),h'(x))]- \\& \E_{x\sim \Pt}[\phi^{*}(\ell(h(x),h'(x))) |.
\end{aligned}
\end{equation}
\end{restatable}
Taking the supremum of $\fhHdiscrepancy$ over $h\in \hypot$, we obtain $\fHdiscrepancy$, and thus $\fhHdiscrepancy(\Ps || \Pt) \leq \fHdiscrepancy(\Ps || \Pt)$.
This bound will be useful when deriving practical algorithms.

\begin{restatable}[\textbf{lower bound}]{lem}{Divergence} \label{lem:upper_lower_bound_fdiv} 
For any two functions $h$,$h'$ in $\hypot$, we have:
\begin{equation}\label{eq:chain}
\begin{aligned}
	|R^{\ell}_S(h,h')-R^{\phi^* \circ \ell}_T(h,h')|
	&\leq \fhHdiscrepancy(\Ps || \Pt) \leq \fHdiscrepancy(\Ps || \Pt)\\ &\leq 
\fdiv(\Ps || \Pt) .
  \end{aligned}
\end{equation}
\end{restatable}

\Cref{lem:upper_lower_bound_fdiv} is fundamental in the derivation of divergence-based generalization bounds for DA. Specifically,   
it bounds the gap between the source and target domains' error functions in terms of the discrepancy between their distributions using $f$-divergences. %
We now show that the $\fhHdiscrepancy$ can be estimated from finite samples.

\begin{restatable}[]{lem}{DivergenceRadamacher} \label{lemma_fhh_from_finite_samples} 
Suppose $\ell:\domainout \times \domainout \to [0,1]$, $\phi^*$  $\textrm{L}$-Lipschitz continuous, and $[0,1]\subset \dom \phi^*$. Let  $\sourcedataset$ and  $\targetdataset$ be two empirical distributions corresponding to  datasets containing  $n$ data points sampled i.i.d.~from  $\Ps$ and $\Pt$, respectively. 
Let us note $\mathfrak{R}$ the Rademacher complexity of a given class of functions, and  $\ell\circ \Hc := \{x\mapsto \ell(h(x), h'(x)): h, h'\in \Hc\}$. $\forall \delta \in (0,1)$, we have with probability of at least $1-\delta$:
\begin{equation}
	\begin{aligned}
&|\fhHdiscrepancy(\Ps||\Pt)-\fhHdiscrepancy(S || T)| \leq 2 \mathfrak{R}_{\Ps}(\ell \circ \hypot)   \\&+ \ 2 \textrm{L} \mathfrak{R}_{\Pt}(\ell \circ \hypot) + 2 \sqrt{(-\log{\delta})/(2 n)}.
\end{aligned}
\end{equation}
\end{restatable}
\updated{
In Lemma \ref{lemma_fhh_from_finite_samples}, we have shown that the empirical $\fhHdiscrepancy$  converges to the true $\fhHdiscrepancy$ discrepancy. It can then be estimated using a set of finite samples from the two distributions.
The gap is bounded by the complexity of the hypothesis class and the number of examples ($n$).
This result will also be important in the derivation of Theorem~\ref{thm:finite_samples_bound}.
}

\subsection{Domain Adaptation: Generalization Bounds} %
\label{sec:f-divergence_bound}
We now provide a novel generalization bound to estimate the error of a classifier in the target domain using the proposed $\fhHdiscrepancy$ divergence and results from the previous section. 
We also provide a generalization Rademacher complexity bound for a binary classifier\footnote{
	Similar bounds can be derived for the multi-class scenario if we let $h: \domainin \times \domainout$ being a score function and $\ell(x,y)=1[\argmax_{\hat y} h(x,\hat y) \neq y ]$ (i.e see \cite{mohri2018foundations} Chapter 9).
} 
based on the estimation of the $\fhHdiscrepancy$ from finite samples. We show that our bound generalizes previous results  
in \Cref{sec:generalizing_ben2010theory}.

\begin{restatable}[\textbf{generalization bound}]{thm}{GenBoundF}\label{thm:general_bound}

Suppose $\ell: \domainout \times \domainout \to [0, 1] \subset \dom \phi^*$. 
Denote $\lambda^*:=R^{\ell}_S( h^*) +  R^{\ell}_T( h^*) , $ and let $h^*$ be the ideal joint hypothesis. 
We have:
\begin{equation}
R^{\ell}_T(h) \leq  R^{\ell}_S( h) +  \fhHdiscrepancy(\Ps || \Pt) +\lambda^*.
\end{equation}

\end{restatable}

The three terms in this upper bound  share similarity with the bounds in~\citet{ben2010theory} and~\citet{zhang19bridging}.
The main difference lies in the discrepancy 
being used to compare the two marginal distributions. 
\citet{ben2010theory} use the $\hypot \Delta \hypot$ divergence  (a reduction of {the TV}), and  \citet{zhang19bridging} use the MDD.
In our case, we use a reduction of a  lower bound estimator of a variational characterization of the general $f$-divergences. 
This generalizes {the TV} (and thus \cite{ben2010theory}) and also includes popular divergences typically used in practice (see \Cref{sec:understanding_dann_mdd}).
Intuitively, the first term in the bound accounts for the source error, the second term corresponds to the discrepancy between the marginal
distributions, and the third term measures the ideal joint hypothesis ($\lambda^*$). 
If $\hypot$
is expressive enough and the labeling functions are similar, this last term could be reduced to a small value. 
The ideal joint hypothesis %
incorporates the notion of adaptability: 
when the optimal hypothesis performs poorly in either domain,
we cannot expect successful adaptation.

\begin{restatable}[\textbf{generalization bound with Rademacher complexity}]{thm}{Radamacher} \label{thm:finite_samples_bound}
Let $\ell: \domainout \times \domainout \to [0,1]$ and $\phi^*$ be $\textrm{L}$-Lipschitz continuous. Let  $\sourcedataset$ and $ \targetdataset$ be two empirical distributions (i.e. datasets containing  $n$ data points sampled i.i.d.~from  $\Ps$ and $\Pt$, respectively). 
Denote $\hat \lambda^*:=  \hat R^{\ell}_S( h^*) + \hat R^{\ell}_T( h^*)$. $\forall \delta \in (0,1)$, we have with probability of at least $1-\delta$: 
\begin{align}
	R^{\ell}_T(h) &\leq \hat R^{\ell}_S( h) + \fhHdiscrepancy(\sourcedataset ||\targetdataset) + \hat \lambda^* \nonumber \\
	 &+ 6\mathfrak{R}_{S}(\ell \circ \hypot)  +  2(1+L) \mathfrak{R}_{T}(\ell \circ \hypot) \nonumber \\ &+ 5 \sqrt{(-\log{\delta})/(2 n)}.
\end{align}

\end{restatable}
\Cref{thm:finite_samples_bound} provides the computation of our generalization bound for a  {binary classifier} 
in terms of the Rademacher complexity of the class $\hypot$. Under the assumption of an ideal joint hypothesis $\hat{\lambda}^*$, the  generalization error can be reduced by jointly minimizing the risk in the source domain, the discrepancy between the two distributions, and regularizing the model to limit  the complexity of the hypothesis class. We take all these into account when deriving practical algorithms in the next sections.

\section{Training Algorithm}
We now exploit the results introduced above to derive a novel and practical domain-adversarial 
algorithm.
We show how our framework for a particular divergence allows us to reinterpret and rectify the original domain-adversarial training method from \citet{ganin2016domain}. Our analysis highlights the differences between our adversarial training algorithm and that from \citet{ganin2016domain}.
Finally, we analyze the use of  $\g$ weighted $f$-divergences. This sheds lights on why the practical objective from \citet{zhang19bridging} outperforms DANN (\citet{ganin2016domain}) and shows how, after a key correction of the latter, the performance gap vanishes.  

\subsection{$f$-Domain Adversarial Learning ($f$-DAL)}
	\label{sec:f-dal}

We now use the theory presented in the previous sections to derive  $f$-DAL, a novel generalized domain adversarial learning framework. 

\textbf{Notation.} 
Let the hypothesis $h$ be the composition of  $h=\hat h \circ g$ (i.e. let $\hypot:=\{\hat h \circ g  : \hat h \in \hat \hypot,  g \in \mathcal{G} \}$ with $\hat \hypot$ another function class) where $g: \domainin \to \latentspace$. 
This can be interpreted as a mapping that  pushes forward the two densities $\psdensity$ and $\ptdensity$ to a representation space $\mathcal{Z}$ where a classifier $\hat h \in \hat \hypot$ operates. Consequently, we denote by $\pzsdensity := g\# \psdensity$ and $\pztdensity := g\# \ptdensity$ the push-forwards  
of the source and target domain densities, respectively. 
\updated{Figure~\ref{fig:fdal_architecture} illustrates the $f$-DAL framework.}

From Theorem~\ref{thm:general_bound}, for adaptation to be possible in the representation space $\latentspace $, we assume the existence of some  $\hat h \in \hat \hypot$ such that the ideal joint risk $\lambda^{*}$ is negligible.  This condition is necessary even if $\pzsdensity= \pztdensity$.
In other words, 
we need both, the difference between $\pzsdensity$ and $\pztdensity$, and the ideal joint risk $\lambda^{*}$ to be small. 
These are both sufficient and necessary conditions. 
We refer the reader to \citet{david2010impossibility} for details on the impossibility theorems for DA. 
Thus, we
 assume that 
there exist some $g \in \mathcal{G}$ and $\hat h^* \in \hat \hypot $, such that 
the ideal joint risk ($\lambda^{*}$)
is negligible. 
These assumptions are ubiquitous in modern DA methods, including SoTA methods \cite{ganin2016domain,long2018conditional,hoffman2018cycada,zhang19bridging}  (sometimes not explicitly mentioned). 
It was recently shown in \citet{zhao2019learning} that for this to be true in the present context,
the label distributions between source and target must be close. 
In \Cref{app:label_shift}, we provide further analysis and experimental results on the robustness of $f$-DAL to label shift. 
Moreover, we show that $f$-DAL can be simply combined with methods that deal with this setting, further boosting their performance. 
We emphasize however that dealing with label shift is outside of the scope of this work.

From Theorem~\ref{thm:general_bound},
 the target risk $\riskTlh$  can be minimized by jointly minimizing the error in the source domain and the discrepancy between the two distributions. Let $y$ be the label of a source data point $z$, 
an optimization objective can be clearly written as:~
\begin{equation}\label{eqn:optimizing_bound}
\min_{\hat h \in \hat \hypot} \E_{z \sim \pzsdensity } [\ell(\hat h(z),y)]  + \textrm{D}^{\phi}_{\hat h,\hat \hypot} (\pzsdensity||\pztdensity).  %
\end{equation}
Here, $\ell$ is a surrogate loss function used to minimize the empirical risk in the source domain.
Under mild assumptions (see Proposition~\ref{prop:optimal_dst}) and the use of Lemma~\ref{lem:upper_lower_bound_fdiv}, %
the minimization problem in \eqref{eqn:optimizing_bound} can be upper bounded (hence replaced) by
the following min-max objective\footnote{$d_{s,t}$ can be seen as an upper bound of the $\fhHdiscrepancy$ discrepancy.}:
\begin{equation}\label{eqn:minmax_optimization_objective}
\min_{\hat h \in \hat \hypot} \max_{ \hat h' \in \hat \hypot} \E_{z \sim \pzsdensity } [\ell(\hat h(z),y)] + d_{s,t} \text{~~~~~where} \\
\end{equation}
\begin{equation*}
d_{s,t}:= \E_{z \sim \pzsdensity } [\hat \ell(\hat h'(z),\hat h(z))]- \E_{z  \sim \pztdensity } [(\phi^* \circ \hat \ell)( \hat h'(z),\hat h(z))].    %
\end{equation*}
We now formalize this result. %
\begin{restatable}[]{prop}{OptSol}\label{prop:optimal_dst}
Suppose $d_{s,t}$ takes the form shown in \eqref{eqn:minmax_optimization_objective} with $\hat \ell(\hat h'(z),\hat h(z)) \to \dom \phi^*$ and 
 that for any $\hat h \in \hat \hypot$ (unconstrained), there exists $\hat h' \in \hat \hypot$ s.t.~
$\hat \ell(\hat h'(z),\hat h(z))=\phi'(\frac{\pzsdensity(z)}{\pztdensity(z)})$ for any $z\in \supp(\pztdensity(z))$, with $\phi'$ the derivative of $\phi$.
The optimal $d_{s,t}$  is $\fdiv (\pzsdistribution || \pztdistribution)$, i.e.~$\max_{\hat h' \in \hat \hypot} d_{s,t} = \fdiv(\pzsdistribution || \pztdistribution)$.
\end{restatable}

If we let the feature extractor $g \in \mathcal{G}$ 
be the one that minimizes both the source error and the discrepancy term, Eq.~\eqref{eqn:minmax_optimization_objective} can be rewritten as:
\begin{equation}\label{eqn:practical_minimax_objective}
\begin{aligned}
\min_{\hat h \in \hat \hypot, g \in \mathcal{G}} \max_{ \hat h' \in \hat \hypot} & \E_{x \sim \psdensity } [\ell(\hat h\circ g ,y)]  
+    \E_{x \sim \psdensity } [\hat \ell(\hat h' \circ g ,\hat h\circ g )] \\&- \E_{x  \sim \ptdensity } [(\phi^* \circ \hat \ell)( \hat h'\circ g ,\hat h\circ g) ] . %
\end{aligned}
\end{equation}

We let $\hat{\ell}(c, b) = a(b_{\argmax\, c})$, where $\argmax\, a$ is the index of the largest element of vector $a$. 
For the choice of $a(.)$, we follow \citet{lee2016fgan} and choose it to be a monotonically increasing function when possible. 
This implies that we choose the domain of $\hat{\ell}$ to be $\R^k \times \R^k$ with $k$ categories.
Intuitively, $\hat h'$ is an auxiliary per-category domain classifier. This makes our framework different from DANN.

\begin{figure}
    \centering
\includegraphics[trim=150 70 120 50, clip,width=0.45\linewidth]{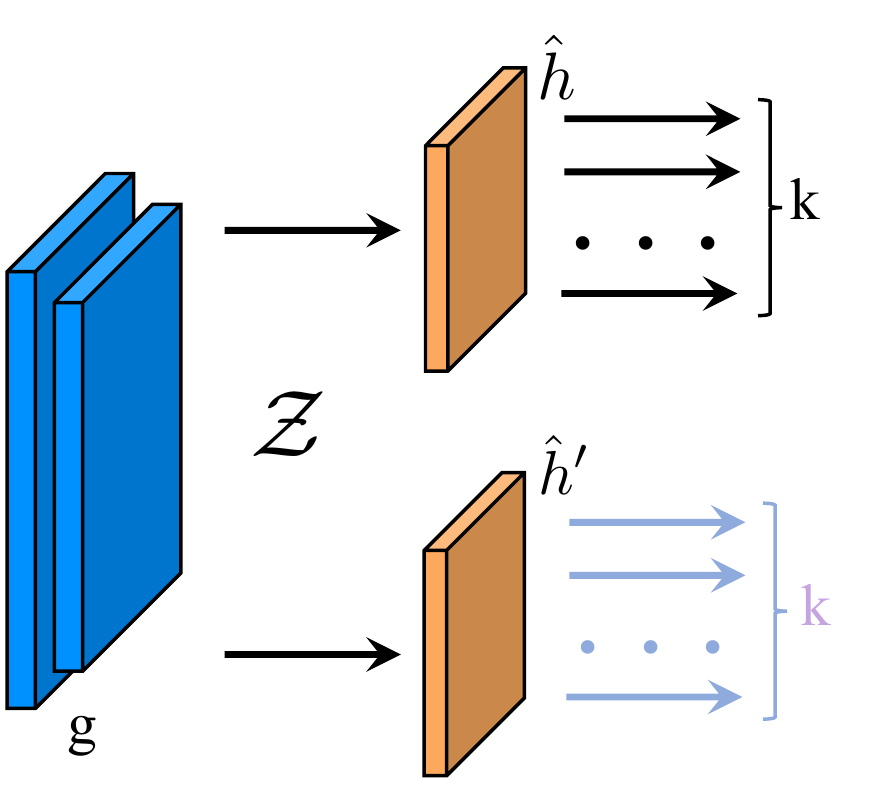} 
\vspace{-3mm}
\caption{
    \updated{\footnotesize \textbf{$f$-DAL framework}. We interpret $h:\domainin \to \domainout$ as the composition of two networks $h=\hat h \circ g$, where $g: \domainin \to \latentspace$ and $\hat h$ is a classifier operating in a representation space $\latentspace$.  Inspired by our bounds, we let $\hat h'$ be a network of the same topology as $\hat h$. This is interpreted as a per-category domain classifier. Unlike us, \citet{ganin2016domain} use a global domain-classifier or “discriminator”.  
    }}
    \label{fig:fdal_architecture}
\end{figure}

\subsection{Revisiting Domain-Adversarial Training (DANN)}\label{sec:revisiting_dann}
The original idea of domain-adversarial training was introduced in \citet{ganin2016domain} and motivated with the theoretical results of
\citet{ben2010theory}. Specifically, the \textit{domain-classifier/regularizer} is inspired by the proxy $\mathcal{A}$-distance \cite{ben2007analysis} which is an approximation of the empirical estimation of the $\hypot \Delta  \hypot$ divergence. 
While it has been shown that under mild assumptions the discrepancy being minimized in DANN corresponds to the JS divergence (see \Cref{sec:understanding_dann_mdd}), the connection between this and the DA theory has not been made clear since, to the best of our knowledge, generalization bounds for DA with $f$-divergences has not been derived. 

In this section, we use our bounds and algorithmic framework to revisit the domain-adversarial training method from \citet{ganin2016domain}.
The analysis shows that while both can be interpreted as minimizing the JS divergence and thus are in line with our theoretical results (\Cref{thm:general_bound}, \Cref{lem:upper_lower_bound_fdiv} and  \Cref{sec:understanding_dann_mdd}), 
DANN ignores the contribution of the source classifier which is not desirable or intuitive.
Experimental results confirm that this apparently subtle difference leads to significant gains (using the same JS divergence, see \cref{tab:fdal_vs_dann_several_datasets,tbl:significancetestfdaljs}).  
To explicitly see this, let us first rewrite the $d_{s,t}$ term in $f$-DAL (\Cref{eqn:practical_minimax_objective}) using the JS divergence (shifted up to a constant that does not alter optimization). We then have
 $\hat{\ell}(h', h ) = \log \sigmoid( {h'}_{\argmax h})    $ and $\phi^*(t) = -\log (1 - e^t)$, where $\sigmoid(x):=\frac{1}{1+\exp(-x)}$ is the sigmoid function.

Plugging all together and rewriting conveniently, we obtain: 
 \begin{equation} \label{fdalJS_Discrepancy}
 \begin{split}
d_{s,t}&=  \E_{x_s \sim \psdensity} \log \sigmoid \circ   \left [\hh'\circ g(x_s)\right ]_{\argmax h  }  \\&+\E_{x_t \sim \ptdensity} \log \left(1-\sigmoid \circ \left [\hh'\circ g(x_t) \right ]_{\argmax h } \right)
\end{split}
 \end{equation}
 which is the resulting $d_{s,t}$ term of $f$-DAL for the JS divergence.
Assuming the output of the source classifier $\hat h$ is constant in terms of the $\argmax$ operator (e.g.~$\hat h=e_i$, with $e_i$ any standard basis vector), we obtain  after manipulation the second part of the expression shown in Equation~(9) in~\citet{ganin2016domain}.
Effectively, this shows that DANN 
ignores the contribution of the source classifier $\hat h$. 
In fact, \textit{it assumes that the output of the source classifier is always constant} (e.g.~$\hat h=e_i$), which is problematic. 
Moreover, the motivation of DANN through the  proxy $\mathcal{A}$-distance ignores the topology/architecture of the discriminator network. This is in contrast with our formulation which suggests that the topology of the per-category domain classifier $\hat h'$ should be identical to that of $\hat h$ since both $\hat h'$, $\hat h \in \hat \hypot$  (\Cref{fig:fdal_architecture}) .

We additionally notice that  $f$-DAL can explain DANN and connect it with the DA theory directly from a JS minimization perspective (i.e. without relying on an approximation of the empirical $\hypot \Delta \hypot$ divergence as in  \citet{ganin2016domain}). This result follows from \Cref{lem:upper_lower_bound_fdiv} and details can be found in  \Cref{sec:understanding_dann_mdd}. This allows us to compare head-to-head $f$-DAL JS vs  DANN, in which scenario   $f$-DAL can be understood as the \textit{corrected/revisited} version of DANN.

\subsection{On $\gamma$-weighted $f$-divergences}\label{sec:gamma_weighted_div}

If we relax the need for $\phi(1) = 0$ in  \Cref{prop:optimal_dst}, 
the new objective only shifts by a constant, e.g., $\max_{\hat h' \in \hat \hypot} d_{s,t} = D_{\hat{\phi}}(\pzsdistribution || \pztdistribution) + \phi(1)$ with $\hat{\phi}(x) := \phi(x) - \phi(1)$. 
By \Cref{lem:perspective} (Appendix~\ref{sec:understanding_dann_mdd}), we can rescale $\phi^*$, and  $\phi$ will change accordingly.
These can be done for the general family of divergences, accommodating a larger family of distributions. 

\textbf{$\gamma$-weighted JS Divergence.} We recall that the objective from MDD \cite{zhang19bridging} (i.e. the one introduced to deal with the practical issues of the MDD discrepancy) corresponds to the $\gamma$-JS divergence (up to a constant that does not alter optimization).
This result gives insight into the big performance gap observed when comparing MDD vs DANN (see \Cref{sec:understanding_dann_mdd}). That gap is due to the fact that DANN considers the output of the source classifier as a constant (see \cref{sec:revisiting_dann}).
After revisiting DANN (\Cref{eqn:practical_minimax_objective} and \Cref{sec:revisiting_dann}), experimental results (\Cref{tab:gamma_w_analysis}) show that the $\gamma$-weighted-JS divergence only performs comparably to the JS divergence with per-dataset extra-tuning of the $\g$ parameter. A statistical analysis shows that this difference in performance (if any) does not justify the expensive introduction of the new hyperparameter $\g$.

\begin{table*}[htbp]
  \begin{minipage}{0.6\linewidth}
  
  \centering
  \caption{Comparison of the $f$-DAL framework vs DANN on different datasets.}
    \resizebox{\textwidth}{!}{%
    \begin{tabular}{lccccc}
    	\toprule
    \multicolumn{1}{c}{\multirow{3}[0]{*}{Method}} & \multicolumn{4}{c}{Datasets}  & \multicolumn{1}{c}{\multirow{3}[0]{*}{Significance}} \\
          & \multicolumn{1}{c}{Toy} & \multicolumn{1}{c}{NLP} & \multicolumn{2}{c}{Vision} &  \\
          & \multicolumn{1}{l}{Digits} & \multicolumn{1}{l}{Amazon Reviews} & \multicolumn{1}{l}{Office-31} & \multicolumn{1}{l}{Office-Home} &  \\
          \midrule
    DANN \cite{ganin2016domain}  & 93.3  & 76.3  & 82.2  & 57.6  & - \\
    \midrule
    $f$-DAL (JS) & \textbf{96.6}  &   80.0    &  88.8     &   66.8    & $\times  \checkmark \checkmark \checkmark$ \\
    $f$-DAL (Pearson $\chi^2$) & 96.3  & \textbf{81.6}  & \textbf{89.2}  & \textbf{68.3}  &  $\times  \checkmark \checkmark \checkmark$  \\
    \midrule
    \bottomrule
    \end{tabular}%
	}
  \label{tab:fdal_vs_dann_several_datasets}%
\end{minipage} 
\begin{minipage}{0.4\linewidth}
  \centering
  \caption{Comparison of $\g$ weighted divergences}
    \resizebox{\textwidth}{!}{%
    \begin{tabular}{ccccc}
    	\toprule
          & $\g$     & Avg Digits & Avg Office-31 & Avg \\
    
    \midrule
    $f$-DAL (JS) & -     & \textbf{96.6}  & 88.8  & 92.7 \\
    $f$-DAL (Pearson $\chi^2$) & -     & 96.3  & \textbf{89.2}  & \textbf{92.8} \\
    \midrule

    \multirow{3}[0]{*}{\shortstack{$f$-DAL($\g$-JS) \\ MDD }} & 2    & 96.0  & 88.1  & 92.0 \\

          & 3    & \textit{96.3}  & 88.5  & 92.4 \\
          & 4    & 96.2  & \textit{88.9}  & 92.5 \\
    \midrule
    \bottomrule
    \end{tabular}%
	}
  \label{tab:gamma_w_analysis}%

\end{minipage}
\end{table*}%

\section{Experimental Results } 
We now experimentally analyze and compare the proposed framework vs previous adversarial methods.
We perform experiments on both toy datasets (digits) and real-world problems (natural language and visual tasks).

\vspace{-2mm}
\subsection{Setup}

\textbf{Digits.}  We evaluate our method on two digits datasets \textbf{MNIST} and \textbf{USPS} with two transfer tasks (M $\to$ U and U $\to$ M). 
We adopt the splits and evaluation protocol from \cite{long2018conditional} which constitute of  
60,000 and 7,291 training images and the standard test set of 10,000 and  2,007 test images for MNIST and USPS, respectively. 

\textbf{Visual Tasks.} We use two visual benchmarks: \textbf{(1)} the \textit{Office-31} dataset \citep{saenko2010adapting} contains 4,652 images and 31 categories, collected from three distinct domains: Amazon \textbf{(A)}, Webcam \textbf{(W)} and DSLR \textbf{(D)}. \textbf{(2)} the \textit{Office-Home} dataset \citep{venkateswara2017Deep} contains 15,500 images from four different domains: Artistic images, Clip Art, Product images, and Real-world images.

\textbf{NLP Tasks}. 
For this task, we consider the Amazon product reviews dataset \cite{blitzer2006domain} which
contains online reviews of different products collected on the Amazon website. We follow the splits and evaluation protocol from \cite{courty2017joint,pmlr-v119-dhouib20b}. We choose 4 of its subsets corresponding to different product categories, namely: books,
dvd, electronics and kitchen (denoted by B, D, E, K, respectively) and leads to 12 domain adaptation tasks of
varying difficulty. The problem is to predict positive (higher than 3 stars) or negative (3 stars or less) notation of reviews. 
For each task, we use predefined sets of 2000 instances
of source and target data samples for training, and keep
4000 instances of the target domain for testing.

\textbf{Baselines.} Our main baseline is DANN \cite{ganin2016domain}. For the JS divergence, our method can be seen as the revisited interpretation of DANN. We then study whether this interpretation based on our bounds correlates well with experimental results.
We also compare with recent methods such as CDAN \cite{long2018conditional} for Digits and JDOT and MADAOT \cite{courty2017joint,pmlr-v119-dhouib20b} for the NLP benchmark. MDD \cite{zhang19bridging} is the $\g$-JS divergence in our framework, we also use it for comparison in visual tasks where results for the method are available.

\begin{figure}  
  \includegraphics[width=.45\textwidth]{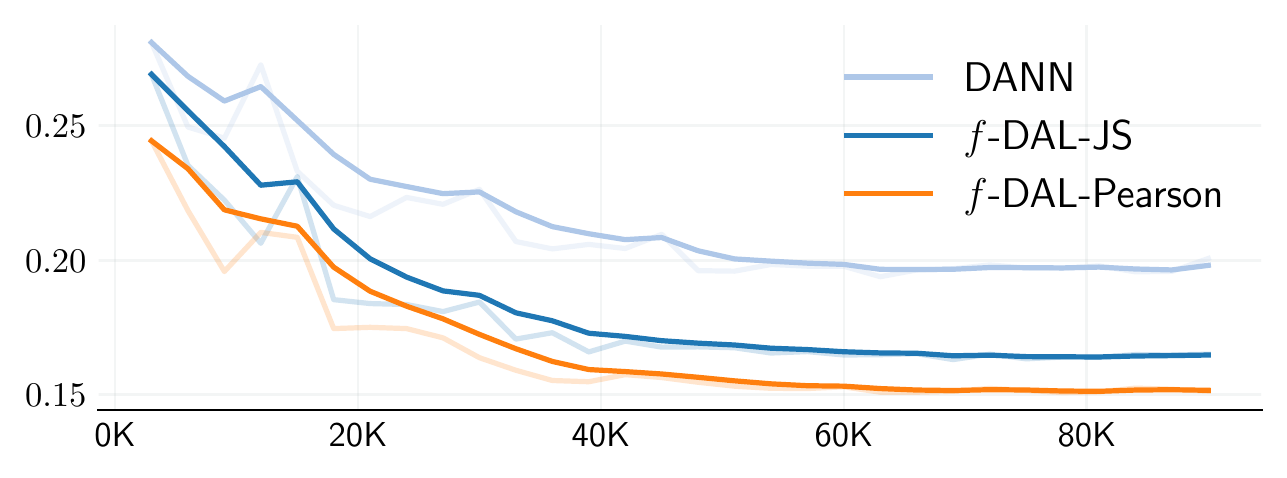}  
  \vspace{-5mm}
  \centering \caption{Target Domain Loss on the Digits Datasets M$\to$ U. }
  \label{fig:mnist_target_domain_loss}
\end{figure}

 \begin{figure}
  \vspace{-3mm}
  \includegraphics[width=0.49\textwidth]{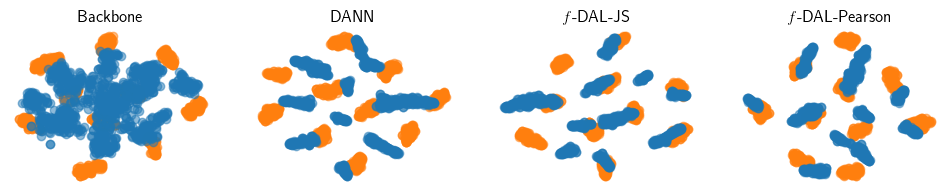}  
  \vspace{-6mm}
  \caption{t-SNE Visualization of the last layer features on the Digits Dataset M$\to$ U. }
  \label{fig:mnist_tsne_vis}
  \vspace{-3mm}
  \end{figure}
  
\textbf{Implementation Details}: We implement our algorithm in PyTorch. 
For the Digits datasets, the implementation details follows \cite{long2018conditional}. 
Thus, the backbone network is LeNet \cite{lecun1998gradient}. The main classifier ($\hat h$) and auxiliary classifier ($\hat h'$)  are both $2$ linear layers  with ReLU non-linearities and Dropout (0.5) in the last layer. 
For the NLP task, we follow the standard protocol from \citet{courty2017joint,ganin2016domain} and use a simple 2-layer model with sigmoid activation function. 
For the visual datasets, we use ResNet-50 \citep{he2016deep} pretrained on ImageNet \citep{deng2009imagenet} as the backbone network.
The main classifier ($\hat h$) and auxiliary classifier ($\hat h'$) are both $2$ layers neural nets with Leaky-ReLU activation functions. We use spectral normalization (SN) as in \cite{miyato2018spectral} only for these two (i.e $\hat h$ and $\hat h'$ ). 
We did not see any transfer improvement by using it.
The reason for this was to avoid gradient issues and instabilities during training for some divergences in the first epochs. For the first two tasks, hyperparameters are determined based on a subset (10\%) of the training set for one task (e.g. M $\to$ U and  B $\to$ D) and kept constant for the others. For the visual tasks, we use the hyperparameters and same training protocol from MDD (\citet{zhang19bridging}). We report the average accuracies over 3 experiments.  Full details are in \Cref{app:more_details_exp_setup}.

\subsection{Experimental Analysis}

	\begin{table*}[t] %
	\addtolength{\tabcolsep}{2pt}
	\centering
	\vspace{-3mm}  
	\centering\caption{Accuracy represented in (\%) with average and standard deviation on the {Office-31} benchmark.} 
	\label{tbl:sota_office31}
	\resizebox{\textwidth}{!}{%
		\begin{tabular}{lccccccc}
			\toprule
			Method                          & A $\rightarrow$ W     & D $\rightarrow$ W     & W $\rightarrow$ D     & A $\rightarrow$ D     & D $\rightarrow$ A     & W $\rightarrow$ A     & Avg           \\
			\midrule
			ResNet-50 \citep{he2016deep}   & 68.4$\pm$0.2          & 96.7$\pm$0.1          & 99.3$\pm$0.1          & 68.9$\pm$0.2          & 62.5$\pm$0.3          & 60.7$\pm$0.3          & 76.1          \\
			DANN \citep{ganin2016domain}   & 82.0$\pm$0.4          & 96.9$\pm$0.2          & 99.1$\pm$0.1          & 79.7$\pm$0.4          & 68.2$\pm$0.4          & 67.4$\pm$0.5          & 82.2          \\
			JAN \citep{long2017deep}         & 85.4$\pm$0.3          & {97.4}$\pm$0.2        & {99.8}$\pm$0.2        & 84.7$\pm$0.3          & 68.6$\pm$0.3          & 70.0$\pm$0.4          & 84.3          \\
			GTA \citep{sankaranarayanan2018generate}       & 89.5$\pm$0.5          & 97.9$\pm$0.3          & 99.8$\pm$0.4          & 87.7$\pm$0.5          & 72.8$\pm$0.3          & 71.4$\pm$0.4          & 86.5          \\
			MCD  \citep{saito2018maximum}   &88.6$\pm$0.2&98.5$\pm$0.1&\textbf{100.0}$\pm$.0&92.2$\pm$0.2&69.5$\pm$0.1&69.7$\pm$0.3&86.5 \\
			CDAN  \cite{long2018conditional}     & 94.1$\pm$0.1          & {98.6}$\pm$0.1 & \textbf{100.0}$\pm$.0 & 92.9$\pm$0.2          & 71.0$\pm$0.3          & 69.3$\pm$0.3          & 87.7          \\
			\midrule

			 $f$-DAL ($\gamma$-JS) / MDD  \citep{zhang19bridging}                &  {94.5}$\pm$0.3 & 98.4$\pm$0.1          &  \textbf{100.0}$\pm$.0 &  {93.5}$\pm$0.2 &  {74.6}$\pm$0.3 &  {72.2}$\pm$0.1 & {88.9} \\
			 \midrule
       Ours ($f$-DAL)  & \textbf{95.4} $\pm$0.7  &   {98.8}$\pm$0.1          & \textbf{100.0}$\pm$.0 &  {93.8 }$\pm$0.4 & \textbf{74.9 }$\pm$1.5 &  {74.2 }$\pm$0.5 & \textbf{89.5}    \\
       \midrule
       Ours ($f$-DAL Pearson) +  Alignment   & 93.4$\pm$0.4  &   \textbf{99.0}$\pm$0.1           & \textbf{100.0}$\pm$.0 &  \textbf{94.8}$\pm$0.6 & {73.6}$\pm$0.2 &  \textbf{74.6}$\pm$0.4 & {89.2}    \\

			\midrule
			\bottomrule
		\end{tabular}

}

\end{table*}

\begin{table*}[ht]
    \vspace{-6mm}
    \caption{Accuracy (\%) on the {Office-Home} benchmark.}
    \label{tbl:sota_officehome}

    \resizebox{\textwidth}{!}{%
    \begin{tabular}{lcccccccccccccc} %
      \toprule
      Method   & Ar$\to$Cl & Ar$\to$Pr & Ar$\to$Rw & Cl$\to$Ar & Cl$\to$Pr  & Cl$\to$Rw & Pr$\to$Ar %
      & Pr$\to$Cl & Pr$\to$Rw & Rw$\to$Ar & Rw$\to$Cl & Rw$\to$Pr & Avg           \\
      \midrule
      ResNet-50 \citep{he2016deep}    & 34.9                   & 50.0                   & 58.0                   & 37.4                   & 41.9                    & 46.2                   & 38.5 
                        & 31.2                   & 60.4                   & 53.9                   & 41.2                   & 59.9                  
       & 46.1          
      \\
      DANN \citep{ganin2016domain}   & 45.6                   & 59.3                   & 70.1                   & 47.0                   & 58.5   
                       & 60.9                   & 46.1                    & 43.7                   & 68.5                   & 63.2                   & 51.8                   & 76.8                   
      & 57.6         
       \\
      JAN \citep{long2017deep}         & 45.9                   & 61.2                   & 68.9                   & 50.4                   & 59.7  
                       & 61.0                   & 45.8                    & 43.4                   & 70.3                   & 63.9                   & 52.4                   & 76.8                   
      & 58.3          
      \\
      CDAN \citep{long2018conditional}     & 50.7                   & 70.6                   & 76.0                   & 57.6                   & 70.0      
                    & 70.0                   & 57.4                    & 50.9                   & 77.3                   & 70.9                   & 56.7                   & 81.6                  
                    & 65.8         
       \\
       \midrule
        $f$-DAL ($\gamma$-JS) / MDD \citep{zhang19bridging}              & {54.9}          &   {73.7}          &  {77.8}          &  {60.0}          &  {71.4}   
                &  {71.8}          &  {61.2}           &  {53.6}          &  {78.1}          &  {72.5}          &  {60.2}          & {82.3}          
              & {68.1}

      \\
      \midrule
      Ours ($f$-DAL)  &54.7 & 71.7& 77.8&  {61.0}&  \textbf{72.6}& 72.2 &  60.8 &  53.4 &  80.0 &  \textbf{73.3} &  \textbf{60.6} &  \textbf{83.8}&  \underline{68.5} \\
      \midrule
       Ours ($f$-DAL - Pearson) + Alignment &  \textbf{56.7} & \textbf{77.0} &  \textbf{81.1} &  \textbf{63.1}&  72.2&  \textbf{75.9} & \textbf{64.5} &  \textbf{54.4}  & \textbf{81.0} & 72.3 &  58.4 &  83.7 & \textbf{70.0} \\

       \midrule
      \bottomrule
    \end{tabular}%
  }
  
\end{table*}

\begin{table*}[h!]
  \centering
  \vspace{-6mm}
  \caption{Accuracy on the Amazon Reviews data sets}
  \resizebox{\textwidth}{!}{%
    \begin{tabular}{lrrrrrrrrrrrrr}
    	\toprule
         Method & \multicolumn{1}{l}{B$\to$D} & \multicolumn{1}{l}{B$\to$E} & \multicolumn{1}{l}{B$\to$K} & \multicolumn{1}{l}{D$\to$B} & \multicolumn{1}{l}{D$\to$E} & \multicolumn{1}{l}{D$\to$K} & \multicolumn{1}{l}{E$\to$B} & \multicolumn{1}{l}{E$\to$D} & \multicolumn{1}{l}{E$\to$K} & \multicolumn{1}{l}{K$\to$B} & \multicolumn{1}{l}{K$\to$D} & \multicolumn{1}{l}{K$\to$E} & \multicolumn{1}{l}{Avg} \\
          \midrule
    JDOTNN \cite{courty2017joint} & 79.5  & 78.1  & 79.4  & 76.3  & 78.8  & 82.1  & 74.9  & 73.7  & 87.2  & 72.8  & 76.5  & 84.5  & 78.7 \\
    MADAOT \cite{pmlr-v119-dhouib20b} & 82.4  & 75.0    & 80.4  & \textbf{80.9}  & 73.5  & 81.5  & \textbf{77.2}  & 78.1  & \textbf{88.1}  & 75.6  & 75.9  & 87.1  & 79.6 \\
   \midrule
    DANN \cite{pmlr-v119-dhouib20b,ganin2016domain}  & 80.6  & 74.7  & 76.7  & 74.7  & 73.8  & 76.5  & 71.8  & 72.6  & 85.0    & 71.8  & 73.0    & 84.7  & 76.3 \\
    \midrule
    Ours ($f$-DAL) & \textbf{84.0}  & \textbf{80.9}  & \textbf{81.4}  &  {80.6}  & \textbf{81.8}  & \textbf{83.9}  & {76.7}  & \textbf{78.3}  & {87.9}  & \textbf{76.5}  & \textbf{79.5}  & \textbf{87.5}  & \textbf{81.6} \\
    \midrule
    \bottomrule
    \end{tabular}%
  \label{tbl:amazon_nlp_shallownet}%
}
\end{table*}%

\begin{table}[h!]
  \centering
  \vspace{-4mm}
  \caption{Accuracy on the Digits datasets}
    
    \begin{tabular}{cccc}
    \toprule
     Method     & M$\to$U  & U$\to$M  & Avg \\
      \midrule
   
    DANN \cite{ganin2016domain}  & 91.8  & 94.7  & 93.3 \\
    CDAN \cite{long2018conditional}  & 93.9  & 96.9  & 95.4 \\
    \midrule
    Ours ($f$-DAL) & \textbf{95.3} &  \textbf{97.3} &  \textbf{96.3} \\
    \midrule
    \bottomrule
    \end{tabular}%
     \vspace{-4mm}
  \label{tbl:digits_table_acc}%
\end{table}%

\begin{figure} %
    \centering
    \includegraphics[width=0.45\textwidth,trim=10 15 25 10, clip]{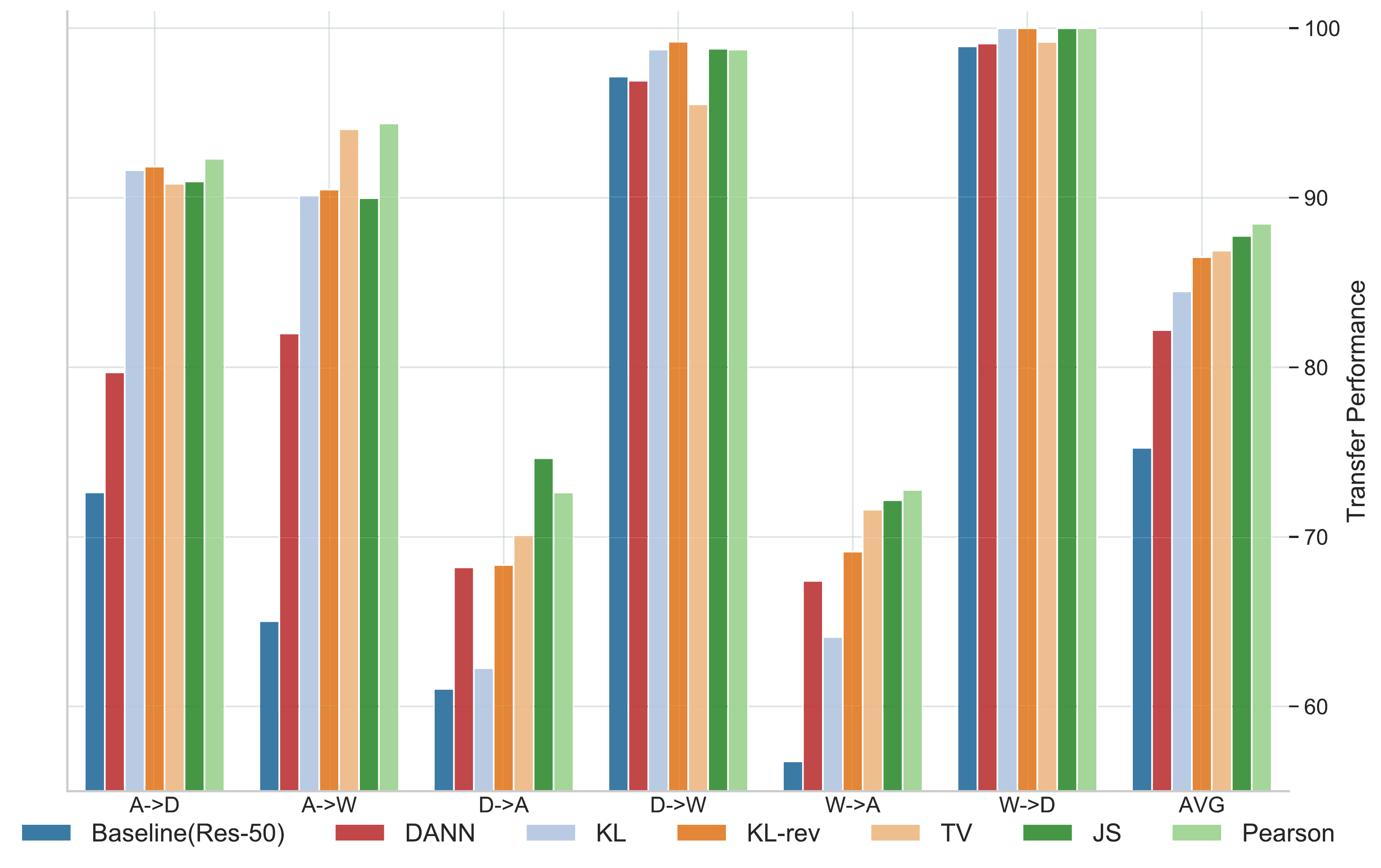}
    \vspace{-3mm}
    \caption{Transfer performance of a model trained using $f$-DAL  for different $f$-divergences and transfer tasks on 
    {Office-31}. Baseline is  {ResNet-50} source only. 
    We show the performance of DANN (Table~\ref{tbl:sota_office31}). When compared with $f$-DAL (JS), a performance boost is observed.
    This is in line with our bounds which suggest the use of a per-category domain classifier vs a discriminator.
}
    \label{fig:the_choice_of_f_div}
\end{figure}

\textbf{Revisited DANN.} 
We now compare the performance of $f$-DAL (JS) vs DANN on the four datasets. 
In this scenario, $f$-DAL (JS) is the corrected version of DANN as discussed in \Cref{sec:revisiting_dann}. 
We can see that $f$-DAL (JS) always outperforms DANN. To further corroborate the statistical significance of this, we conducted a two sided Wilcoxon signed rank test. 
With the exception of the Digits datasets (for which performance is beyond 90\%),  $f$-DAL (JS) is statistically significantly better than DANN (5\% significance, 95\% confidence, \Cref{tbl:significancetestfdaljs}). For the digits dataset, we provide training losses in the target domain in  Fig. \ref{fig:mnist_target_domain_loss} and t-SNE \cite{maaten2008visualizing} visualizations of the last layer input (perplexity=30) in Fig. \ref{fig:mnist_tsne_vis}. $f$-DAL (JS) converges faster and the resulting features are also  better aligned.

\textbf{Comparing $f$-divergences.} 
We compare the performance of $f$-divergences on \textit{Office-31}.
Specifically, we  evaluate the model on the six combinations of  transfer tasks with different divergences.
All hyperparameters are kept constant for all divergences in this experiment. 
As shown in Figure~\ref{fig:the_choice_of_f_div}, the JS and Pearson $\chi^2$ divergences achieve the best results,  with 
the \textit{Pearson $\chi^2$ achieving the best overall result among all the transfer tasks on this benchmark}. 
This is also the case for the Digits, NLP and Office-Home datasets.
It is worth noting that this divergence was never used before to learn invariant representations in the context of  DA.
The excellent performance of $\chi^2$ is also reminiscent of histogram-based (visual) bags of words representations that were shown to work better with $\chi^2$ distances than with $\ell_2$ and $\ell_1$ distances for image and text classification tasks \citep{li2013sign}.

\textbf{Comparing $\g$-weighted divergences.} We now investigate the significance of introducing the hyper-parameter $\g$ to define the $\g$-weighted divergences.  
We compare in~\Cref{tab:gamma_w_analysis} the performance of using $\g$-JS vs JS and Pearson in two benchmarks: (1) Digits and (2) Office-31. The $\g$-JS divergence only outperforms the JS after tuning the hyperarameter $\g$. The difference is only of $0.1 \%$ in average in the Office-31 dataset giving a p-val=0.89 using the Wilcoxon signed rank test. 
This means that after correction with our framework DANN/$f$-DAL-JS is as good as $\g$-JS without additional hyperparameter tuning.  In general, we found the use Pearson $\chi^2$ divergence gives slightly better numerical results. 

\textbf{Training Dynamics.}  Fig. \ref{fig:mnist_target_domain_loss} and Fig. \ref{fig:mnist_lhat_val} illustrate the target loss curves and the values of $\hat \ell$ for JS and Pearson, respectively. In both cases our framework converges faster and achieves lower cost (see \Cref{fig:mnist_target_domain_loss}). \Cref{fig:mnist_lhat_val} illustrates the value of $\hat \ell$  for both source and target where  $\hat \ell \approx \phi'(1)=0 $, which implies $\pzsdensity \approx \pztdensity$ (\Cref{prop:optimal_dst}) as desired.  It is worth noting that while this is true in both cases, domain invariance is achieved faster (almost after the first epoch) with the Pearson $\chi^2$. This could also give intuition about the noticeable performance gap while using this divergence. 

  \begin{figure}
  \includegraphics[width=.5\textwidth]{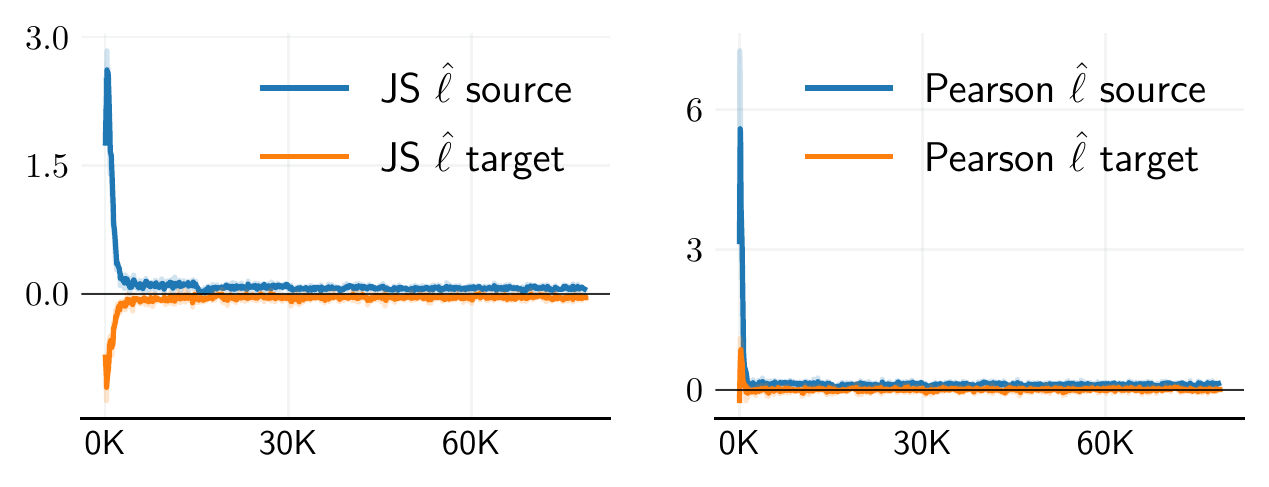} 
  \vspace{-10mm}
  \caption{Values of $\hat \ell(\hat h',\hat h)$ for source and target on Digits M$\to$ U. 
  $\hat \ell \approx \phi'(1)=0 $, which implies $\pzsdensity \approx \pztdensity$ (see \Cref{prop:optimal_dst})}
  \label{fig:mnist_lhat_val}
  \vspace{-5mm}
  \end{figure}

\textbf{Results.} We compare our  method vs.~recent state-of-the-art domain adversarial approaches in \Cref{tbl:digits_table_acc,tbl:amazon_nlp_shallownet,tbl:sota_office31,tbl:sota_officehome}.
\textit{Ours} in the tables correspond to $f$-DAL using the Pearson $\chi^2$ divergence, with the exception of D $\to$ W and D $\to A$ in \Cref{tbl:sota_office31}, and Ar $\to$ Pr in \Cref{tbl:sota_officehome} where we use JS divergence.  A detailed version of these with every divergence's performance can be found in \Cref{app:add_experimental_results}.
In all cases, our approach outperforms previous methods, including MDD which is also included in our framework (\Cref{sec:gamma_weighted_div}), and requires tuning of the hyperparameter $\g$. 
What is most impressive is that, unlike our approach, some methods listed in the tables can be interpreted as DANN + additional techniques to improve their performance (i.e. CDAN).
It would be interesting to see if these techniques still introduce gains after correcting DANN (i.e. $f$-DAL JS) or if they were necessary because of the disconnect between theory and algorithms. 

\textbf{Improving $f$-DAL with Sampling-Based Alignment.}
In this experiment, we show that if the distance between the label marginals is not negligible $f$-DAL is still effective and can simply be combined with SoTA methods that deal with the label shift such as \citet{pmlr-v119-jiang20d}. 
We refer to this in \Cref{tbl:sota_office31,tbl:sota_officehome} as  ``+Alignment.''
For this experiment, we follow the setting from \citet{pmlr-v119-jiang20d} but replace the adversarial method for $f$-DAL-Pearson. We also remove their masking scheme as we did not find it necessary with $f$-DAL.   
Clearly, in the Office-31 dataset (\Cref{tbl:sota_office31}) the distance between the label marginals is not significantly different and we did not see any improvement by introducing implicit alignment. 
This is in contrast with~\Cref{tbl:sota_officehome} (Office-Home dataset) where  our method notably benefits from the sampling-based alignment scheme.
This again showcases the versatility of $f$-DAL.
We refer to~\Cref{app:label_shift} for more details and experiments on label-shift.

\section{Related Work}
\textbf{Theory.}
The domain adaptation problem has been rigorously investigated in \cite{ben2007analysis,ben2010theory,mansour2009domain,zhao2019learning,zhang19bridging} where a classifier’s target error is bounded 
in terms of its source error and the divergence between the two domains.
We propose a measure of discrepancy between distributions based 
on 
a  variational characterization of $f$-divergences. 
Our method includes the $\hypot \Delta \hypot$-divergence as a particular case but also  other divergences used in practice.
Moreover, our bounds based on $f$-divergences allow us to connect theory and practical algorithms
without 
surrogate objectives.

\textbf{Domain-Adversarial Algorithms.} 
\citet{ganin2016domain} introduced domain-adversarial training with insights from 
\citet{ben2010theory}. 
 This algorithm has been heavily adopted in the context of neural networks
\cite{ long2018conditional,hoffman2018cycada, zhang19bridging}.  
We propose a general adversarial framework for the family of $f$-divergences based on our bounds.
We show how to correct the training algorithm from \citet{ganin2016domain}, and how to incorporate a large family of $f$-divergences.
We explain why MDD 
 \cite{zhang19bridging} outperforms \citet{ganin2016domain} and show how the gap vanishes after correction.

\textbf{Variational $f$-divergences.} \citet{nguyen2010estimating} propose a derivation of  the variational characterization of $f$-divergences that was later used for GANs \cite{lee2016fgan}. These were used in the context of DA in an example in \citet{wu2019domain} to
rewrite the domain-regularizer from \citet{ganin2016domain}.
We derive $f$-divergence based generalization bounds from which we derive an algorithmic framework different from \citet{ganin2016domain}. 
Our analysis shows how to correct DANN.
Morever, experimental results showing the performance of $f$-divergences in the context of domain-adversarial learning has not been provided.

\section{Conclusions}
We have provided a novel perspective on the domain-adversarial problem by deriving a general domain adaptation framework.
Our bounds are based on a variational characterization of $f$-divergences and recover the theoretical results from seminal works as a special case, and also support divergences typically used in practice.
We have derived a general algorithmic framework that is practical for neural networks. 
It allows us to reinterpret and correct the original domain-adversarial training method.
We also show 
 through large-scale experiments that several $f$-divergences can be used to minimize the discrepancy between source and target domains. We showed that some divergences that do not require additional techniques and/or hyperparameter tuning can help achieve state-of-the-art performance.

\textbf{Acknowledgements.} We would like to thank Rafid Mahmood, Mark Brophy and the anonymous reviewers for helpful discussions and feedback on earlier versions of this manuscript.   

\small
\FloatBarrier
\bibliography{main}
\bibliographystyle{icml2021}
\newpage
\normalsize
\appendix

\newpage
\onecolumn
\icmltitle{Supplementary Material\\ $f$-Domain-Adversarial Learning: Theory and Algorithms}

\section{Divergences between probability measures} \label{app:divergences}

As explained above, the difference term between source and target domains is important in bounding the target loss. We now provide more details about the $\Hc\Delta\Hc$-divergence and $f$-divergences that are used to compare both domains. 

\paragraph{$\Hc\Delta\Hc$-divergence} The $\Hc$-divergence is a restriction of total variation. For binary classification, define $I(h) := \{ \xv\in \Xc: h(\xv) = 1\}$, then the $\Hc$-divergence between two measures $\mu$ and $\nu$ given the hypothesis class $\Hc$ is \citep{ben2010theory}:
\be\label{eq:hdiv}
d_\Hc(\mu, \nu) = 2\sup_{h\in \Hc} |\mu(I(h)) - \nu(I(h))|.
\en
Define $\Hc\D\Hc := \{h\oplus h': h, h'\in \Hc\}$ ($\oplus$: XOR), then $d_{\Hc\D\Hc}(\mu, \nu)$ can be used to bound the difference between the source and target errors. $\Hc\D\Hc$ divergence has been extended to general loss functions \citep{mansour2009domain} and marginal disparity discrepancy \citep{zhang19bridging}.

\paragraph{$f$-divergence} Given two measures $\mu$ and $\nu$ with $\mu \ll \nu$ ($\mu$ absolute continuous w.r.t.~$\nu$), the $f$-divergence $D_\phi(\mu || \nu)$ is defined as \citep{csiszar1967information, ali1966general}:
\be\label{eq:fdiv}
D_\phi(\mu \parallel \nu) = \int \phi\left(\frac{d\mu}{d\nu}\right)d\nu,
\en
where ${d\mu}/{d\nu}$ is known as the Radon--Nikodym derivative \citep[e.g.][]{billingsley2008probability}. Assume $\phi$ is convex and lower semi-continuous, then from the Fenchel--Moreau theorem, $\phi^{**} = \phi$, with $\phi^*$ known as the Fenchel conjugate of $\phi$:
\be\label{eq:conjugate}
\phi^*(\yv) = \sup_{\xv\in \dom \phi} \langle \xv, \yv \rangle - \phi(\xv),     
\en
which is convex since it is a supremum of an affine function. In order for $\xv$ to take the supremum, it is necessary and sufficient that $\yv\in \partial \phi(\xv)$ using the stationarity condition. Therefore, with \eqref{eq:fdiv} and \eqref{eq:conjugate}, $D_\phi(\mu \parallel \nu)$ can be written as:
\be
D_\phi(\mu \parallel \nu) = \sup_{T\in \mathcal{T}}\E_{X\sim \mu}[T(X)] - E_{Z\sim \nu}[\phi^*(T(Z))],
\en
where $\mathcal{T} = \{T: T\mbox{ is a measurable function and } T: \Xc \to \dom \phi^*\}$. 
In practice we restrict $\mathcal{T}$ to a subset as in \Cref{def:disc1}.
For different choices of $\phi$ see \Cref{tbl:choices_f_diver_supp}.  

\cite{nguyen2010estimating} derive a general variational method to estimate $f$-divergences given only samples. 
\cite{lee2016fgan} extend their method from merely estimating a divergence for a fixed model
to estimating model parameters. 
While our method  builds on this variational formulation, we use it in the context of domain adaptation.

\begin{table*}[tb] 

\begin{center}
\small
\resizebox{\textwidth}{!}{%
  \begin{tabular}{lllllll}
  \toprule
  Divergence & $\phi(x)$  & $\phi^*(t)$ & $ \phi'(1) $& $g(x)$ %
  \\ \midrule
  MDD
  & $x \log \frac{\g x}{1 + \g x} + \frac{1}{\g}\log \frac{1}{1 + \g x}$
  & $-\log ( 1 - e^t)/\g$
  & $\log \frac{\g}{1+\g}$
  & $\log x$\\
    Kullback-Leibler (KL)
  & $x \log x$
  & $\exp(t-1)$
  & $1$
  & $x$
  \\
  Reverse KL (KL-rev)
  & $\scalebox{0.9}[1.0]{-} \log x$
  & $\scalebox{0.9}[1.0]{-}1-\log (-t)$
  & $-1$
  &  $-\exp x$
  \\
  Jensen-Shannon (JS)
  & $\scalebox{0.9}[1.0]{-} (x+1) \log \frac{1+x}{2}+x \log x$
  & $\scalebox{0.9}[1.0]{-} \log(2-e^t)$
  & $0$
  & $\log \frac{2}{1+\exp(-x)}$
  \\
   Pearson $\chi^2$ 
  & $(x-1)^2$ 
  & $t^2/4 + t$ 
  & $0$ 
  & $x$
  \\
  Squared Hellinger (SH)
  & $(\sqrt{x}-1)^{2}$
  & $\frac{t}{1-t}$
  & $0$
  &  $1 - \exp x$
  \\
  $\g$-weighted Pearson $\chi^2$ 
  & $(\g x-1)^2/\g$ 
  & $(t^2/4 + t)/\g$ 
  & $0$ 
  & $x$
  \\
  Neynman $\chi^2$ 
  & $\frac{(1-x)^2}{x}$ 
  & $2 - 2 \sqrt{1-t}$ 
  & $0$ 
  & $1 - \exp x$ 
  \\
  $\g$-weighted total variation 
  & $\frac{1}{2\g}|\g x-1|$
  & $(t/\g) \one_{-1/2\leq t \leq 1/2}$
  & $[-1/2, 1/2]$
  & $\frac{1}{2}\tanh x$
  \\
  
  Total Variation (TV) 
  & $\frac{1}{2}|x-1|$
  & $ \one_{-1/2\leq t \leq 1/2}$
  & $[-1/2, 1/2]$
  & $\frac{1}{2}\tanh x$\\
  \bottomrule

  \end{tabular}
}%
\end{center}
\vspace{-3mm}
\caption{  Popular $f$-divergences, their conjugate functions and choices of $g$. We take $\hat{l}(a, b) = g(b_{\argmax\, a})$.
}

\label{tbl:choices_f_diver_supp}
\end{table*}

\section{Proofs}

In this section, we provide the proofs for the different theorems and lemmas:

\TotalVar*
\begin{proof}
Rewriting the target loss we have:
\begin{align*}
    \centering
    \riskTlh  &=\riskTlh  - R^{\ell}_S(h,\labelingf_t) + R^{\ell}_S(h,\labelingf_t)  -  \riskSlh + \riskSlh,\\ 
     &\leq \riskSlh + |\riskSlh-R^{\ell}_S(h,\labelingf_t)| +|\riskTlh-R^{\ell}_S(h,\labelingf_t)| \\
\end{align*}
where:
\begin{align*}
    |\riskSlh-R^{\ell}_S(h,\labelingf_t)| &= |R^{\ell}_S(h,\labelingf_s)-R^{\ell}_S(h,\labelingf_t)| \\
    &=|\E_{x \sim \Ps}[|h(x)-\labelingf_t(x)|-|h(x)-\labelingf_s(x)|]|\\
    &\leq \E_{x \sim \Ps}[|\labelingf_t(x)-\labelingf_s(x)|]
\end{align*}
and: %
\begin{align*}
    |\riskTlh-R^{\ell}_S(h,\labelingf_t)| &= |R^{\ell}_T(h,\labelingf_t)-R^{\ell}_S(h,\labelingf_t)| \\
    &\leq \int |\ptdensity(x)-\psdensity(x)|\cdot |h(x)-\labelingf_t(x)|dx \\
    &\leq  \int | \big ( \frac{\ptdensity(x)}{\psdensity(x)}-1 \big ) \psdensity(x)| dx = D_{\phi}(\Ps||\Pt)
\end{align*}
with $\phi(x)=|x-1|$ which represents the total divergence. 
\end{proof}

\Divergence*
 \begin{proof}
\begin{align}
\fHdiscrepancy(\Ps || \Pt) &= \sup_{h\in \Hc} \fhHdiscrepancy(\Ps || \Pt) \geq \fhHdiscrepancy(\Ps || \Pt) \\ %
&= \sup_{h'\in \Hc} |E_{x\sim \Ps} [\ell(h(x),h'(x))] - \E_{x\sim \Pt}[\phi^{*}(\ell(h(x),h'(x)))]| \\ %
& \geq |E_{x\sim \Ps} [\ell(h(x),h'(x))] - \E_{x\sim \Pt}[\phi^{*}(\ell(h(x),h'(x)))]| \\  %
  &=|R^{\ell}_S(h,h')-R^{\phi^* \circ \ell}_T(h,h')|.
\end{align}

For the rightmost inequality in \eqref{eq:chain}, it is well-known that $f$-divergence $D_\phi$ is nonnegative \citep[e.g.][]{sason2016f}, and thus
\be
D_\phi(\Ps \| \Pt) = \sup_{T\in \mathcal{T}} |\E_{x\sim \Ps} T(x) - \E_{x\sim \Pt} \phi^*(T(x))|.
\en
Restricting $\mathcal{T}$ to $\mathcal{\hat T}$ as in \Cref{def:disc1} we obtain $D_\phi(\Ps \| \Pt) \geq \fHdiscrepancy(\Ps || \Pt)$. 
 \end{proof}

\DivergenceRadamacher*
\begin{proof} 
 For reference, we refer the reader to Chapter 3 of \cite{mohri2018foundations}. Using the notations of $R$ and $\hat{R}$ that represent the true and empirical risks, we have:
\begin{align}
\fhHdiscrepancy(\Ps || \Pt) - \fhHdiscrepancy( \sourcedataset ||\targetdataset ) 
&= \sup_{h' \in \hypot} \{ | R^{\ell}_S(h,h')-R^{\phi^* \circ \ell}_T(h,h')| \} \\ \nonumber 
&- \sup_{h' \in \hypot} \{ |\hat R^{\ell}_S(h,h')-\hat R^{\phi^* \circ \ell}_T(h,h')  |\} \\ \nonumber
& \leq \sup_{h' \in \hypot} | |R^{\ell}_S(h,h')-R^{\phi^* \circ \ell}_T(h,h')| - |\hat R^{\ell}_S(h,h') -  \hat R^{\phi^* \circ \ell}_T(h,h')| | \\ \nonumber
& \leq \sup_{h' \in \hypot} |  R^{\ell}_S(h,h')-R^{\phi^* \circ \ell}_T(h,h')  -  \hat R^{\ell}_S(h,h') +  \hat R^{\phi^* \circ \ell}_T(h,h')|  \\ \nonumber
& = \sup_{h' \in \hypot}   | R^{\ell}_S(h,h')-  \hat R^{\ell}_S(h,h')|+|    R^{\phi^* \circ \ell}_T(h,h') -\hat R^{\phi^* \circ \ell}_T(h,h')  |  \\ \nonumber
& \leq 2 \mathfrak{R}_{\Ps}(\ell \circ \hypot) + \sqrt{\frac{\log{\frac{1}{\delta}}}{2 n}} +2 \mathfrak{R}_{\Pt}(\phi^* \circ \ell \circ \hypot) + \sqrt{\frac{\log{\frac{1}{\delta}}}{2 n}}
\end{align}

where: 
$| R^{\ell}_S(h,h')- \hat R^{\ell}_S(h,h') | \leq 2 \mathfrak{R}_{\Ps}(\ell \circ \hypot) + \sqrt{\frac{\log{\frac{1}{\delta}}}{2 n}}$ (Theorem 3.3 of \cite{mohri2018foundations}).
Similarly, by Talagrand's lemma (Lemma 5.7 and Definition 3.2 of \cite{mohri2018foundations}) we have:
$\mathfrak{R}_{\Pt}(\phi^* \circ \ell \circ \hypot) \leq \textrm{L} \mathfrak{R}_{\Pt}(\ell \circ \hypot)$, with $\phi^*\circ \ell\circ \Hc := \{x\mapsto \phi(\ell(h(x), h'(x))): h, h'\in \Hc\}$.
\end{proof}

\GenBoundF*
\begin{proof}
We first introduce the following lemma for our proof:
\begin{lem}\label{lem:greater_than_id}
For any function $\phi$ that satisfies $\phi(1) = 0$ we have $\phi^*(t) \geq t$ where $\phi^*$ is the Fenchel conjugate of $\phi$.
\end{lem}

\begin{proof}
From the definition of Fenchel conjugate, $\phi^*(t) = \sup_{x\in \dom \phi} (x t - \phi(x)) \geq t - \phi(1) = t$.
\end{proof}

\begin{align}
R^{\ell}_T(h,\labelingf_t) &\leq R^{\ell}_T(h,h^*) + R^{\ell}_T(h^*,\labelingf_t) & \text{(triangle inequality $\ell$) }\\
&= R^{\ell}_T(h,h^*) + R^{\ell}_T(h^*,\labelingf_t) - R^{ \ell}_S(h,h^*)+R^{   \ell}_S(h,h^*) \\
&\leq  R^{\phi^*\circ \ell}_T(h,h^*) - R^{ \ell}_S(h,h^*) + R^{  \ell}_S(h,h^*) +R^{\ell}_T(h^*,\labelingf_t) & \text{(Lemma~\ref{lem:greater_than_id})} \\
&\leq  |R^{\phi^* \circ \ell}_T(h,h^*) - R^{ \ell}_S(h,h^*)| + R^{  \ell}_S(h,h^*) +R^{\ell}_T(h^*,\labelingf_t)  \\
&\leq \fhHdiscrepancy(\Ps || \Pt) + R^{  \ell}_S(h,h^*) +R^{\ell}_T(h^*,\labelingf_t)  & \text{(Lemma~\ref{lem:upper_lower_bound_fdiv})} \\  
&\leq \fhHdiscrepancy(\Ps || \Pt) +  R^{  \ell}_S(h,\labelingf_s) + \underbrace{ R^{  \ell}_S(h^*,\labelingf_s) +R^{\ell}_T(h^*,\labelingf_t)}_{\lambda^*}. 
\end{align}~
\end{proof}

\Radamacher*

\begin{proof}
We show in the following that: \begin{align}
R^{\ell}_T(h) &\leq  
   \hat R^{\ell}_S( h) + \fhHdiscrepancy(\sourcedataset ||\targetdataset) %
  + \hat \lambda^*_\phi \\ 
  &~~~~~~~~~~~~~~~~+ 6\mathfrak{R}_{S}(\ell \circ \hypot)  +  2(1+L) \mathfrak{R}_{T}(\ell \circ \hypot) + 5 \sqrt{(-\log{\delta})/(2 n)}. %
\end{align}

This follows from Theorem~\ref{thm:general_bound} where: $
R^{\ell}_T(h) \leq  R^{\ell}_S( h) +  \fhHdiscrepancy(\Ps || \Pt) +  R^{\ell}_S( h^*) +  R^{\ell}_T( h^*)$.
We also have:  $| R^{\ell}_D(h)- \hat R^{\ell}_D(h) | \leq 2 \mathfrak{R}_{D}(\ell \circ \hypot) + \sqrt{\frac{\log{\frac{1}{\delta}}}{2 n}}$ (Theorem of 3.3 \cite{mohri2018foundations}). From Lemma~\ref{lemma_fhh_from_finite_samples},  $\fhHdiscrepancy(\Ps||\Pt)   \leq
2 \mathfrak{R}_{\Ps}(\ell \circ \hypot)  +2 \textrm{L} \mathfrak{R}_{\Pt}(\ell \circ \hypot) + 2 \sqrt{\frac{\log{\frac{1}{\delta}}}{2 n}}$. Plugging in and rearranging  gives the desired results. 
\end{proof}

\OptSol*
\begin{proof}
We first rewrite from the definition of $d_{s,t}$ in \eqref{eqn:minmax_optimization_objective}:
\begin{align}
d_{s,t}&=\E_{z \sim \pzsdensity } [\hat \ell(\hat h'(z),\hat h(z))]- \E_{z  \sim \pztdensity } [(\phi^* \circ \hat \ell)( \hat h'(z),\hat h(z))] \\
&=  \int [\pzsdensity(z) \hat \ell(\hat h'(z),\hat h(z))- \pztdensity (z) (\phi^* \circ \hat \ell)( \hat h'(z),\hat h(z))] dz \\
&=  \int \pztdensity(z) \left[\frac{\pzsdensity(z)}{\pztdensity(z)}\hat \ell(\hat h'(z),\hat h(z))- (\phi^* \circ \hat \ell)( \hat h'(z),\hat h(z))\right] dz.
\end{align}
Maximizing w.r.t $h'$ and assuming $\hat \hypot$ is unconstrained we have: $\frac{\pzsdensity(z)}{ \pztdensity (z) } \in (\partial \phi^*) (\hat \ell( \hat h'(z),\hat h(z))$ for any $z\in \supp(\pztdensity)$. %
From the definition of Fenchel conjugate we have:
$$x \in \partial \phi^*(t) \iff  \phi(x)+\phi^*(t) = xt \iff \phi'(x) = t.$$
Plugging $x = {\pzsdensity(z)}/{\pztdensity (z)}$ and $t = \ell(\hat h'(z),\hat h(z))$ we obtain
$\ell(\hat h'(z),\hat h(z)) = \phi'({\pzsdensity(z)}/{\pztdensity (z)})$.
Hence, from the definition of $f$-divergences (\Cref{def:fdivergence}) and its variational characterization (eq.~\ref{eq:sup_measurable_function}), we write:~
\be \max_{\hat h' \in \hat \hypot} d_{s,t}= D_\phi(\Ps^z || \Pt^z).\en
\end{proof}

\section{Connection to previous frameworks}\label{sec:understanding_dann_mdd}

In this appendix we show that $f$-DAL encompasses previous frameworks on domain adaptation, including $\Hc\Delta\Hc$-divergence, DANN \citep{ganin2016domain} and MDD \citep{zhang19bridging}.
\subsection{$\Hc\Delta\Hc$-divergence}
\label{sec:generalizing_ben2010theory}
We now show that Theorem~\ref{thm:general_bound} generalizes the bound proposed in \cite{ben2010theory}. Let the pair $\{\phi(x),\phi^*(t)\}=\{ \frac{1}{2}|x-1|,t \}$ for $ t \in [0,1]$, such that $\fhHdiscrepancy = \TVhHdiscrepancy$ and $\sup_{h\in \hypot}\TVhHdiscrepancy=\TVHdiscrepancy =\frac{1}{2} d_{\hypot \Delta \hypot}$, with $d_{\hypot \Delta \hypot}$ defined in \cite{ben2010theory} (see also \eqref{eq:hdiv}).
\Cref{thm:general_bound} gives us that $R^{\ell}_T(h) \leq  R^{\ell}_S( h) + \frac{1}{2} d_{\hypot \Delta \hypot}  +\lambda^*$, recovering Theorem 2 of \cite{ben2010theory}. 

\subsection{DANN formulation and JS divergence}

 The DANN formulation by \citet{ganin2015unsupervised} can also be incorporated in our framework if one takes 
 $\hat{\ell}(\hat h' \circ g (x), e_1) = \log \sigmoid(e_1 \cdot  \hat h' \circ g (x))    $ and $\phi^*(t) = -\log (1 - e^t)$, where $\sigmoid(x):=\frac{1}{1+\exp(-x)}$ is the sigmoid function, and $e_1$ corresponds to the standard basis vector. Reinterpreting $\hat h ':= e_1 \cdot \hat h'$, sustituting and computing $d_{s,t}$ we obtain:

 \begin{align}
d_{s,t}&=  \E_{x_s \sim \psdensity} \log \sigmoid \circ  \hh'\circ g(x_s)  +\E_{x_t \sim \ptdensity} \log \left(1-\sigmoid \circ \hh'\circ g(x_t) \right) \label{DANN_Discrepancy} \\
& =  - \left[ \E_{x_s \sim \psdensity} \log\frac{1}{ \sigmoid \circ  \hh'\circ g(x_s)}  +\E_{x_t \sim \ptdensity} \log \frac{1}{  1-\sigmoid \circ \hh'\circ g(x_t)  }   \right ], 
 \end{align}
which is equivalent with the second part of the expression show in equation~9 in~\cite{ganin2016domain}.

Effectively, this formulation ignores the contribution of the source classifier $\hat h'$. 
In fact, \textit{it assumes the output of the source classifier is always constant} (e.g $\hat h=e_1$).
Notice that this is corrected in $f$-DAL where $\hat{\ell}(a, b) = g(b_{\argmax \, a})$. 
We experimentally also observed that this formulation leads to an inferior performance.
Nonetheless, the following proposition shows that under the assumption of an optimal domain classifier $\hat h'$, $d_{s,t}$ achieves JS-divergence (up to a constant shift), which upper bounds the $\JShHdiscrepancy$. 

\begin{prop} %
Suppose $d_{s,t}$ follows the form of eq.~\ref{DANN_Discrepancy} and $\hat h’$ is the optimal domain classifier which is unconstrained, then $\max_{\hat h'} d_{s,t}  = \textrm{D}_{\textrm{JS}}(S||T) - 2\log2$. 
\end{prop}
\begin{proof}
For simplicity in the notation let $\hat h ':= \sigmoid \circ (e_1 \cdot \hat h')$, rewritting eq.~\ref{DANN_Discrepancy} we have:
\begin{align}
d_{s,t}(\hh',g)&= \int_{\mathcal{Z}}  \pzsdensity(z) \log \hh'(z) +  \pztdensity(z) \log (1- \hh'(z)) dz.
\end{align}
By taking derivatives and finding the optimal $\hh^{*}(z)$, we get : 
$
h^{*}(z)=\frac{\pzsdensity (z)}{\pzsdensity (z)+\pztdensity(z)}$.

By plugging $ \hh^{*}(z)$ into \eqref{DANN_Discrepancy}, rearranging, and using the definition of the Jensen-Shanon (JS) divergence,  we get the desired result. 
\end{proof}

It is worth noting that the additional negative constant $-2\log 2$ does not affect the optimization.

\subsection{MDD formulation and $\g$-weighted JS divergence}

Now let us demonstrate how our $f$-DAL framework incorporates MDD naturally. Suppose $\phi^*(t) = -\frac{1}{\g}\log(1 - e^t)$ and $\hat{\ell}(\hh(z), \hh'(z)) = \log \hh'(z)_{\argmax \, \hh(z)}$. We retrieve the following result as in \citet{zhang19bridging}:

\begin{prop}[\textbf{\citet{zhang19bridging}}]
Suppose $d_{s,t}$ takes the form of MDD, i.e, 
\be\label{eq:dst_zhang}
\g d_{s,t} = \g \E_{z\sim \pzsdensity} \log \hh'(z)_{\argmax \, \hh(z)} + \E_{z\sim \pztdensity} \hh(z)\cdot \log ( 1 - \hh'(z)_{\argmax \, \hh(z)}).
\en
With unconstrained function class $\hat\Hc$, the optimal $d_{s,t}$ satisfies:
\be\label{eq:max_mdd}
\max_{\hh'} \g d_{s,t} = (\g + 1) {\rm JS}_{\g}(\pzsdensity\| \pztdensity) + \g \log \g - (\g + 1) \log (\g + 1),
\en
where ${\rm JS}_{\g}(\pzsdensity\| \pztdensity)$ is $\gamma$-weighted Jensen--Shannon divergence \citep{huszar2015not, lee2016fgan}:
\be\label{eq:mdd_dst}
{\rm JS}_{\g}(\pzsdensity\| \pztdensity) = \frac{\g}{\g + 1}{\rm KL}(\pzsdensity\| \frac{\g \pzsdensity + \pztdensity}{\g + 1}) + \frac{1}{\g + 1} {\rm KL}(\pztdensity\| \frac{\g \pzsdensity + \pztdensity}{\g + 1}).
\en
\end{prop}
We remark that when $\g = 1$, ${\rm JS}_{\g}(\pzsdensity\| \pztdensity)$ is the original Jensen--Shannon divergence. One should also note the the additional negative constant $\g \log \g - (\g + 1) \log (\g + 1)$, which attributes to the negativity of MDD, does not affect the optimization. 

$\phi^*(t) = -\frac{1}{\g}\log(1 - e^t)$ can be considered by rescaling the $\phi^*$ for the usual JS divergence (see \Cref{tbl:choices_f_diver_supp}). In general we can rescale $\phi^*$ for any $f$-divergence with the following lemma:

\begin{restatable}[\citet{boyd2004convex}]{lem}{}\label{lem:perspective}
For any $\lambda > 0$, the Fenchel conjugate of $\lambda \phi$ is $(\lambda \phi)^*(t) = \lambda \phi^*(t/\lambda)$, with $\dom (\lambda\phi)^* = \lambda \dom \phi^*$. 
\end{restatable}

\subsection{ Revisiting MCD \cite{saito2018maximum} }
Let's now use $f$-DAL to revisit MCD. This will allow us to understand the cause of the performance gap. 
For example,  MCD(86.5) vs Ours (89.5) on Office-31. Moreover, it will show us how to improve MCD.
Let $\hat \ell(c,b)=|c-b|$ in \Cref{eqn:practical_minimax_objective}, and choose $\phi$ to be the TV (\Cref{tbl:choices_f_diver}). We have:

\begin{align}
\min_{\hat h \in \hat \hypot, g \in \mathcal{G}} \max_{ \hat h' \in \hat \hypot}   \ R_s [ \hat h\circ g]  
+    \E_{ \psdensity } [| \hat h' \circ g - \hat h\circ g | ] - \E_{ \ptdensity } [ | \hat h'\circ g -\hat h\circ g |]  %
\end{align}

where $\hat \ell$
should be in $[-0.5,0.5]$ to satisfy requirements on 
$\phi^*$ (\Cref{tbl:choices_f_diver}).
Comparing this  with MCD
we can see \textbf{3} key  differences. \textbf{1)} MCD ignores the second term based on assumptions, further requires careful initialization for  $\hat h,\hat h'$.
 \textbf{2)} The max operator in their case  {goes over $\hat h$ and $\hat h'$. 
 This makes optimization harder (see \citet{zhang19bridging}).
We do not need this because our bounds are based on $D^\phi_{h,\hypot}\leq D^\phi_{\hypot}$ (definitions \ref{def:disc1} and \ref{def:disc2}, \Cref{lem:upper_lower_bound_fdiv}). \textbf{3)} The restriction on the $\hat \ell(c,b)$ is not taken into account (should be re-weighted or the act. function  follow Tab 1). 
As mentioned in MCD (Eq. 9), $I[c\neq b]$ is similar, but in this context not the same as $|c-b|$. 
 Thus, 1,2,3 could explain the difference in performance  86.5 vs Ours (89.5).
 We believe 
 using these recommendations on MCD could lead to a powerful algorithm but we defer that to further work.

\section{Additional Experimental Results}\label{app:add_experimental_results}
\begin{table*}[h] %
  \addtolength{\tabcolsep}{2pt}
  \centering
  \vspace{-3mm}  
  \centering\caption{Accuracy represented in (\%) with average and standard deviation on the {Office-31} benchmark.} 
  \label{tbl:sota_office31_additional_results}
  \resizebox{\textwidth}{!}{%
    \begin{tabular}{lccccccc}
      \toprule
      Method                          & A $\rightarrow$ W     & D $\rightarrow$ W     & W $\rightarrow$ D     & A $\rightarrow$ D     & D $\rightarrow$ A     & W $\rightarrow$ A     & Avg           \\
      \midrule
      ResNet-50 \citep{he2016deep}   & 68.4$\pm$0.2          & 96.7$\pm$0.1          & 99.3$\pm$0.1          & 68.9$\pm$0.2          & 62.5$\pm$0.3          & 60.7$\pm$0.3          & 76.1          \\
      DANN \citep{ganin2016domain}   & 82.0$\pm$0.4          & 96.9$\pm$0.2          & 99.1$\pm$0.1          & 79.7$\pm$0.4          & 68.2$\pm$0.4          & 67.4$\pm$0.5          & 82.2          \\
      JAN \citep{long2017deep}         & 85.4$\pm$0.3          & {97.4}$\pm$0.2        & {99.8}$\pm$0.2        & 84.7$\pm$0.3          & 68.6$\pm$0.3          & 70.0$\pm$0.4          & 84.3          \\
      GTA \citep{sankaranarayanan2018generate}       & 89.5$\pm$0.5          & 97.9$\pm$0.3          & 99.8$\pm$0.4          & 87.7$\pm$0.5          & 72.8$\pm$0.3          & 71.4$\pm$0.4          & 86.5          \\
      MCD  \citep{saito2018maximum}   &88.6$\pm$0.2&98.5$\pm$0.1&{100.0}$\pm$.0&92.2$\pm$0.2&69.5$\pm$0.1&69.7$\pm$0.3&86.5 \\
      CDAN  \cite{long2018conditional}     & 94.1$\pm$0.1          & {98.6}$\pm$0.1 & {100.0}$\pm$.0 & 92.9$\pm$0.2          & 71.0$\pm$0.3          & 69.3$\pm$0.3          & 87.7          \\
      \midrule

       $f$-DAL ($\gamma$-JS) / MDD  \citep{zhang19bridging}   &  {94.5}$\pm$0.3 & 98.4$\pm$0.1          &  {100.0}$\pm$.0 &  {93.5}$\pm$0.2 &  {74.6}$\pm$0.3 &  {72.2}$\pm$0.1 & {88.9} \\
       $f$-DAL (JS)  & {93.0}$\pm$1.4   &  {98.8}$\pm$0.1  &  {100.0}$\pm$.0 &92.8$\pm$0.4 & {74.9}$\pm$1.5 & 73.3$\pm$0.1   & 88.8  \\

        $f$-DAL (Pearson $\chi^2$)   & {95.4}$\pm$0.7  &   {98.4}$\pm$0.2          & {100.0}$\pm$.0 &  {93.8}$\pm$0.4 & {73.5}$\pm$1.1 &  {74.2}$\pm$0.5 & \textbf{89.2}    \\
 \midrule
        $f$-DAL($\gamma$-JS) / MDD + Alignment \citep{pmlr-v119-jiang20d} & 90.3$\pm$0.2 & 98.7$\pm$0.1 & 99.8$\pm$.0  & 92.1$\pm$0.5 & 75.3$\pm$0.2 & 74.9$\pm$0.3 & 88.8 \\
        $f$-DAL (Pearson $\chi^2$) + Alignment   & 93.4$\pm$0.4  &   {99.0}$\pm$0.1           & {100.0}$\pm$.0 &  {94.8}$\pm$0.6 & {73.6}$\pm$0.2 &  {74.6}$\pm$0.4 & \textbf{89.2}    \\
      \midrule
      \bottomrule
    \end{tabular}

}

\end{table*}

  \begin{table*}[ht]
    \vspace{-6mm}
    \caption{Accuracy (\%) on the {Office-Home} benchmark.}
    \label{tbl:sota_officehome_additional}

    \resizebox{\textwidth}{!}{%
    \begin{tabular}{lcccccccccccccc} %
      \toprule
      Method   & Ar$\to$Cl & Ar$\to$Pr & Ar$\to$Rw & Cl$\to$Ar & Cl$\to$Pr  & Cl$\to$Rw & Pr$\to$Ar %
      & Pr$\to$Cl & Pr$\to$Rw & Rw$\to$Ar & Rw$\to$Cl & Rw$\to$Pr & Avg           \\
      \midrule
      ResNet-50 \citep{he2016deep}    & 34.9                   & 50.0                   & 58.0                   & 37.4                   & 41.9                    & 46.2                   & 38.5 
                        & 31.2                   & 60.4                   & 53.9                   & 41.2                   & 59.9                  
       & 46.1          
      \\
      DANN \citep{ganin2016domain}   & 45.6                   & 59.3                   & 70.1                   & 47.0                   & 58.5   
                       & 60.9                   & 46.1                    & 43.7                   & 68.5                   & 63.2                   & 51.8                   & 76.8                   
      & 57.6         
       \\
      JAN \citep{long2017deep}         & 45.9                   & 61.2                   & 68.9                   & 50.4                   & 59.7  
                       & 61.0                   & 45.8                    & 43.4                   & 70.3                   & 63.9                   & 52.4                   & 76.8                   
      & 58.3          
      \\
      CDAN \citep{long2018conditional}     & 50.7                   & 70.6                   & 76.0                   & 57.6                   & 70.0      
                    & 70.0                   & 57.4                    & 50.9                   & 77.3                   & 70.9                   & 56.7                   & 81.6                  
                    & 65.8         
       \\
       \midrule
        $f$-DAL ($\gamma$-JS) / MDD \citep{zhang19bridging}              & {54.9}          &   {73.7}          &  {77.8}          &  {60.0}          &  {71.4}   
                &  {71.8}          &  {61.2}           &  {53.6}          &  {78.1}          &  {72.5}          &  {60.2}          & {82.3}          
              & {68.1}

      \\
      $f$-DAL (JS)   & 53.7&71.7  &76.3  & 60.2  & 68.4  & 69.0  & 60.2   & 52.6    & 76.9    & 71.4    & 59.0  & 81.8  & 66.8 &  \\

         $f$-DAL (Pearson $\chi^2$)   &54.7 & 69.4& 77.8&  {61.0}&  72.6& 72.2 &  60.8 &  53.4 &  80.0 &  73.3 &  60.6 &  {83.8}& 68.3  \\

       \midrule
       $f$-DAL($\gamma$-JS) / MDD + Alignment \citep{pmlr-v119-jiang20d} & 56.2   & 77.9  & 79.2   & 64.4  &  73.1  & 74.4   & 64.2 &   54.2 & 79.9 & 71.2 & 58.1 & 83.1 & 69.5\\
       $f$-DAL (Pearson $\chi^2$) + Alignment &  56.7 & 77.0 &  81.1 &  63.1&  72.2&  75.9 & 64.5 &  54.4  & 81.0 & 72.3 &  58.4 &  83.7 & \textbf{70.0}    \\

      \bottomrule
    \end{tabular}%
  }
  
\end{table*}

\begin{table*}[h!]
  \centering
  \vspace{-6mm}
  \caption{Accuracy on the Amazon Reviews data sets}
  \resizebox{\textwidth}{!}{%
    \begin{tabular}{lrrrrrrrrrrrrr}
      \toprule
         Method & \multicolumn{1}{l}{B$\to$D} & \multicolumn{1}{l}{B$\to$E} & \multicolumn{1}{l}{B$\to$K} & \multicolumn{1}{l}{D$\to$B} & \multicolumn{1}{l}{D$\to$E} & \multicolumn{1}{l}{D$\to$K} & \multicolumn{1}{l}{E$\to$B} & \multicolumn{1}{l}{E$\to$D} & \multicolumn{1}{l}{E$\to$K} & \multicolumn{1}{l}{K$\to$B} & \multicolumn{1}{l}{K$\to$D} & \multicolumn{1}{l}{K$\to$E} & \multicolumn{1}{l}{Avg} \\
          \midrule
    JDOTNN \cite{courty2017joint} & 79.5  & 78.1  & 79.4  & 76.3  & 78.8  & 82.1  & 74.9  & 73.7  & 87.2  & 72.8  & 76.5  & 84.5  & 78.7 \\
    MADAOT \cite{pmlr-v119-dhouib20b} & 82.4  & 75    & 80.4  & 80.9  & 73.5  & 81.5  & 77.2  & 78.1  & 88.1  & 75.6  & 75.9  & 87.1  & 79.6 \\
   \midrule
    DANN \cite{pmlr-v119-dhouib20b,ganin2016domain}  & 80.6  & 74.7  & 76.7  & 74.7  & 73.8  & 76.5  & 71.8  & 72.6  & 85.0    & 71.8  & 73.0    & 84.7  & 76.3 \\
    \midrule
    $f$-DAL (JS) & 83.2 & 78.8 &  80.4 & 80.2 &  79.4 &  82.9 &  72.3  & 76.3 &  87.8  & 74.7 &  78.5 &  87.0 & 80.1  \\
    $f$-DAL (Pearson $\chi^2$) & {84.0}  & {80.9}  & {81.4}  & {80.6}  & {81.8}  & {83.9}  & {76.7}  & {78.3}  & {87.9}  & {76.5}  & {79.5}  & {87.5}  & \textbf{81.6} \\
    \midrule
    \bottomrule
    \end{tabular}%
  \label{tab:amazon_nlp_shallownet}%
}
\end{table*}%

\begin{table}[h!]
  \centering
  \caption{Accuracy on the Digits datasets}
    
    \begin{tabular}{cccc}
    \toprule
     Method     & M$\to$U  & U$\to$M  & Avg \\
      \midrule
   
    DANN \cite{ganin2016domain}  & 91.8  & 94.7  & 93.3 \\
    CDAN \cite{long2018conditional}  & 93.9  & 96.9  & 95.4 \\
    \midrule
    $f$-DAL (JS) &   {95.3}     &  {98.0}      &  \textbf{96.6}  \\
    $f$-DAL (Pearson $\chi^2$) & {95.3} &  {97.3} &  {96.3} \\
    \midrule
    \bottomrule
    \end{tabular}%
  \label{tab:addlabel}%
\end{table}%

\begin{table}[htbp]
  \centering
  \caption{p-values Significance Test (Wilcoxon signed rank test)}
    \begin{tabular}{lcccc}
      \toprule
          & \multicolumn{1}{l}{Digits} & \multicolumn{1}{l}{NLP} & \multicolumn{1}{l}{Office-31} & \multicolumn{1}{l}{Office-Home} \\
    \midrule
    Avg DANN & 93.3  & 76.3  & 82.2  & 57.6 \\
    Avg $f$-DAL JS & 96.6  & 80.1  & 88.8  & 66.8 \\
    p-val & 0.5   & 0.0025  & 0.031  & 0.0025  \\
    \midrule
    \bottomrule
    \end{tabular}%
  \label{tbl:significancetestfdaljs}%
\end{table}%

\subsection{Experimental results with others $\gamma$-shifted divergences }
In this section, we show experiments on the  Digits Benchmark (Avg on 3 runs) for a shifted $\gamma$-Pearson $\chi^2$. 
We follow \Cref{sec:gamma_weighted_div} and let $\hat \phi(x)=\phi(x)-\gamma x$. 
Results shown in \Cref{tab:shiftedPearson} are similar to those obtained for the $\gamma$-JS (\Cref{tab:gamma_w_analysis}), for which our test showed no significance to have $\gamma$.  We also conducted experiments for the other modality, e.g. NLP data, with  $\gamma$-JS. Similarly, we observed results are not significant  wrt JS($\gamma$=3, Avg=80.4) and slightly worse than Pearson.

\begin{table}[h]
  \centering
  \caption{$\gamma$-shifted Pearson $\chi^2$ Digits Benchmark. }
    \begin{tabular}{rr}
    \toprule
    \multicolumn{1}{l}{$\gamma$} & \multicolumn{1}{l}{Avg Digits} \\
     \midrule
     -  & \textit{96.3} \\
     \midrule
    2     & 96.2 \\
    3     & \textbf{96.4} \\
    4     & 96.3 \\
    \midrule
    \bottomrule
    \end{tabular}%
  \label{tab:shiftedPearson}%
\end{table}%

\begin{figure}[h]
    \centering
    \includegraphics[width=0.25\textwidth,trim=80 80 100 85, clip]{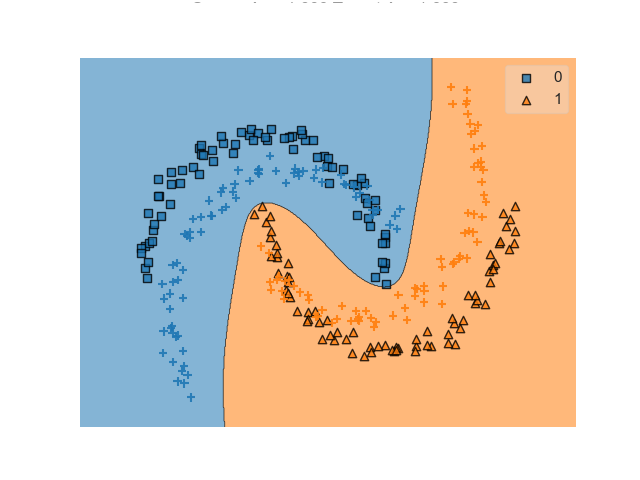}
    \vspace{-3mm}
    \caption{\small \textit{Domain Adaptation}.  A learner trained on  abundant labeled data (marked as squares, colors are categories) is expected to perform well in the target domain (marked as \textbf{+}). Decision boundaries correspond to a 2-layers neural net trained using $f$-DAL.}
    \label{fig:concept_figure}
\end{figure}

\vspace{-3mm}
\subsection{Robustness to Label Shift} \label{app:label_shift}

In this section, we compare the robustness to label-shift of $f$-DAL-JS vs DANN on the digits benchmark. 
Specifically, we consider the task M$\to$ U and artificially generate different version of the target dataset where data-points are re-sampled in terms of its classes. This way we can have control over the JS divergence between the label distribution (i.e $JS(P_s(y)||P_t(y))$) and compare at different levels.
\Cref{fig:exp_label_shift} shows the results.
Firstly, we can observe that both methods performance degrades as the distance between label distributions increases.
This is an expected behavior in DA, and can also be explained with our theory. 
For example, as this distance increases, the term $\lambda^*$ in \Cref{thm:general_bound} simply increases, and thus this cannot be assumed to be negligible.
To explicitly see  why, we refer the reader to \citet{zhao2019learning} where the authors derived a lower bound for joint risk. 
It is important to also have in mind that $\lambda^*$ incorporates the notion of adaptability.
That is, if the optimal hypothesis performs poorly in either domain, adaptation is simply not possible and thus assumptions are need it.
Secondly, from the figure, we can also see our method is more robust to label-shift than DANN. 
Indeed, we fit linear regression models to highlight the trend and show the value of the slope in each case.
The performance comparison is noticeable. 
We emphasize the aim of this experiment is to showcase the robustness of $f$-DAL-JS vs DANN when label-shift exists.
Our method does not propose any additional correction or term to deal with this and doing so (i.e dealing explicitly with label-shift) is out-of-the-scope of this work.
Our algorithm follows the common assumption stated on adversarial DA methods  and let $\lambda^*$ to be negligible.
We believe the better performance of $f$-DAL-JS vs DANN under label-shift is just a consequence of directly connecting theory and algorithm.
We additionally show $f$-DAL can be perfectly combined with methods that deal with label shift such as Implicit Alignment (i.e \citet{pmlr-v119-jiang20d}) (\Cref{tbl:sota_office31_additional_results,tbl:sota_officehome_additional}).
Indeed, doing so leads to SoTA results on the Office-Home dataset (\Cref{tbl:sota_officehome_additional}).
This again showcases the versatility of $f$-DAL. 

\begin{figure}[!h]
    \centering
    \includegraphics[width=0.5\textwidth]{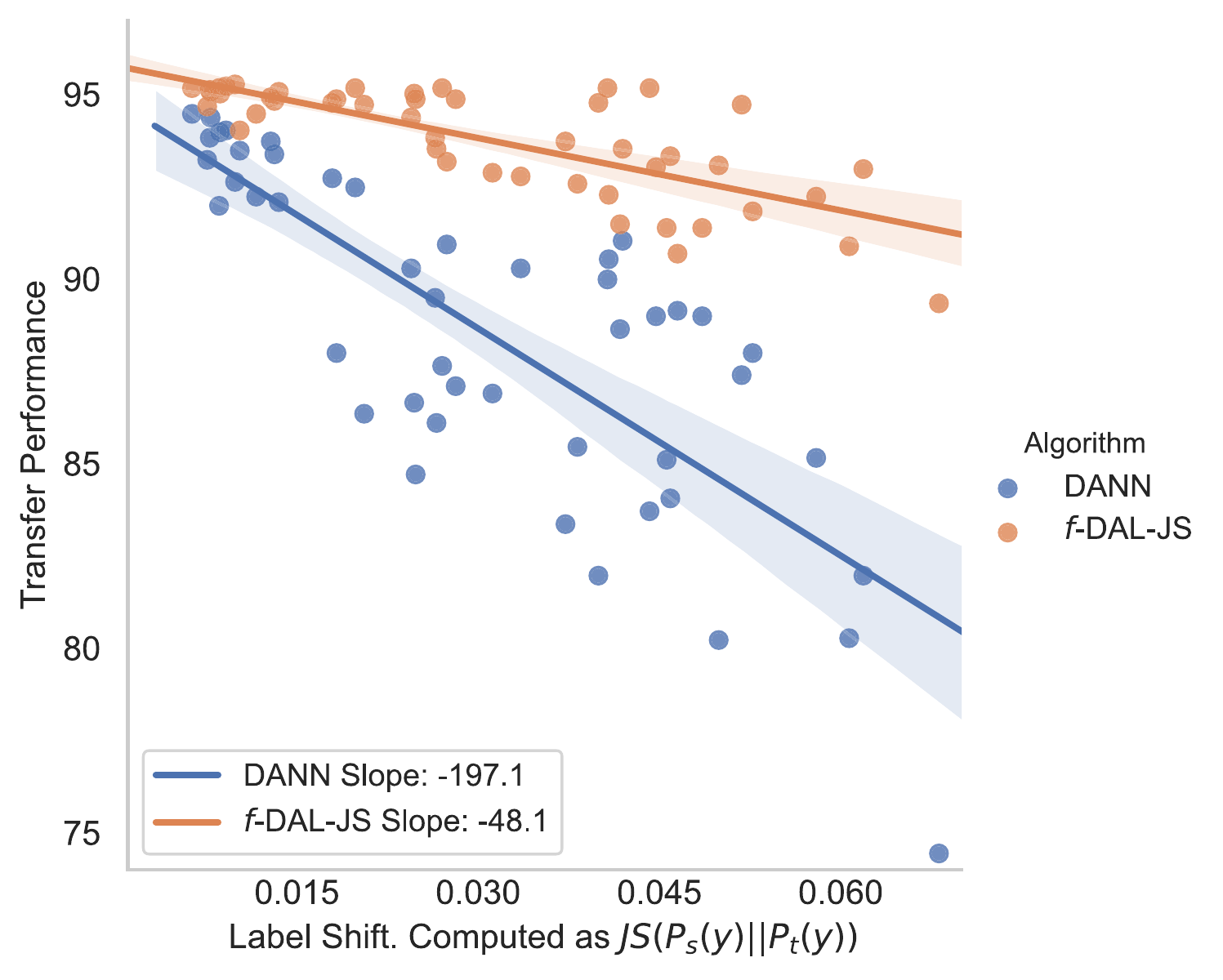}
    \caption{\small \textit{Robustness to Label Shift $f$-DAL-JS vs DANN}. 
    The x-axis represents the Jensen-Shanon distance between the label distributions.
    We can observe that $f$-DAL-JS is more robust to label shift than DANN.
    Linear regression models are fit to highlight the trend(slope is also shown).
     (Dataset M $\to$ U).}
    \label{fig:exp_label_shift}
    \vspace{-4mm}
\end{figure}

\newpage
\section{More Details on Experimental Setup} \label{app:more_details_exp_setup}

Our algorithm is implemented in PyTorch.
For the Digits datasets, the implementation details follows \citet{long2018conditional}. 
Thus, the backbone network is LeNet \cite{lecun1998gradient}. The main classifier ($\hat h$) and auxiliary classifier ($\hat h'$)  are both $2$ linear layers  with Relu non-linearities and Dropout (0.5) in the last layer. We train for 30 epochs, the optimizer  is SGD with Nesterov Momentum (momentum 0.9, batch size 128), the learning rate is 0.01. The regularization term for the discrepancy is set to 0.5 and the GRL coefficient set to 0.6.  We use a weight decay coefficient of 0.002.
Hyperparameters follow  closely the ones used by \citet{long2018conditional}, if some differ slightly, they were determined in a subset(10\%) of the training set of the task M$\to$U and kept constant for the other task. We use three different seeds (i.e 1,2,3) and report the average over the runs.

For the NLP task, we follow the standard protocol from \citet{courty2017joint,ganin2016domain} and use simple 2-layer model with sigmoid activation function. Thus, the main classifier ($\hat h$) and auxiliary classifier ($\hat h'$) are a simple linear layer with BN.
We train for 10 epochs, the optimizer  is SGD with Nesterov Momentum (momentum 0.9, batch size 16), the learning rate is 0.001. We use three different seeds (i.e 1,2,3) and report the average over the runs. The regularization term for the discrepancy is set to 1 and the GRL coefficient set to 0.1.  We use a weight decay coefficient of 0.002.
Hyper-parameters are empirically determined in a subset(10\%) of the training set of the task  (B$\to D$ ) and kept constant for the others.

For the visual datasets, we use ResNet-50 \citep{he2016deep} pretrained on ImageNet \citep{deng2009imagenet} as the backbone network.
The main classifier ($\hat h$) and auxiliary classifier ($\hat h'$) are both $2$ layers neural nets with Leaky-Relu activation functions. We use spectral normalization (SN) as in \cite{miyato2018spectral} only for these two (i.e $\hat h$ and $\hat h'$ ). 
We did not see any transfer improvement by using it.
The reason for this was to avoid gradient issues and instabilities during training for some divergences in the first epochs. 
We use the hyperparams and same training protocol from MDD (\citet{zhang19bridging} and CDAN (\citet{long2018conditional}). We report the average accuracies over 3 experiments.  

Experiments are conducted on NVIDIA Titan V (Digits, NLP) and V100 (Visual Tasks) GPU cards.

\end{document}